\documentclass{article}

\RequirePackage{silence}
\usepackage{microtype}
\usepackage{graphicx}
\usepackage{subfigure}
\usepackage{booktabs} 

\usepackage{hyperref}

\usepackage[accepted]{icml2023}

\usepackage{amsmath}
\usepackage{amssymb}
\usepackage{mathtools}
\usepackage{amsthm}
\usepackage{thmtools}

\usepackage[capitalize,noabbrev]{cleveref}

\theoremstyle{plain}
\newtheorem{theorem}{Theorem}[section]
\newtheorem{proposition}[theorem]{Proposition}
\newtheorem{lemma}[theorem]{Lemma}
\newtheorem{corollary}[theorem]{Corollary}
\theoremstyle{definition}
\newtheorem{definition}[theorem]{Definition}

\theoremstyle{remark}

\usepackage[textsize=tiny]{todonotes}

\usepackage{centernot}  
\usepackage{xfrac}      
\usepackage{bbm}        

\usepackage{enumitem}   

\usepackage{parskip}    

\def\bbN{\mathbb{N}}

\def\calA{\mathcal{A}}
\def\calB{\mathcal{B}}
\def\calC{\mathcal{C}}
\def\calD{\mathcal{D}}
\def\calE{\mathcal{E}}

\def\calI{\mathcal{I}}

\def\calK{\mathcal{K}}
\def\calL{\mathcal{L}}
\def\calM{\mathcal{M}}
\def\calN{\mathcal{N}}
\def\calO{\mathcal{O}}

\def\calS{\mathcal{S}}

\def\calU{\mathcal{U}}
\def\calV{\mathcal{V}}
\def\calW{\mathcal{W}}
\def\calX{\mathcal{X}}
\def\calY{\mathcal{Y}}
\def\calZ{\mathcal{Z}}

\usepackage{mathtools} 

\newcommand{\abs}[1]{\left\vert #1\right\vert}

\newcommand{\rbr}[1]{\left(#1\right)}
\newcommand{\sbr}[1]{\left[#1\right]}
\newcommand{\cbr}[1]{\left\{#1\right\}}
\newcommand{\abr}[1]{\left\langle#1\right\rangle}

\def\norm#1{\|#1\|}

\def\argmax{\mathop{\rm arg\,max}}
\def\argmin{\mathop{\rm arg\,min}}

\def\half{\frac 1 2}

\newcommand{\R}{\mathbb{R}}

\newcommand{\into}{\rightarrow}

\newcommand{\w}{w}

\newcommand{\xk}{x_k}

\newcommand{\xbar}{\xbar{w}}

\newcommand{\bi}{{b_i}}
\newcommand{\wi}{w_\bi}
\newcommand{\bwi}{\bar{w}_\bi}
\newcommand{\vi}{v_\bi}
\newcommand{\ci}{c_\bi}
\newcommand{\zi}{z_\bi}
\newcommand{\Xbi}{X_\bi}
\newcommand{\act}{{\calA_{\lambda}}}

\newcommand{\equi}{{\calE_{\lambda}}}

\newcommand{\acts}{{\calA_{\lambda}^{*}}}

\newcommand{\wa}{w_\act}

\newcommand{\we}{w_\equi}

\newcommand{\Xa}{X_\act}
\newcommand{\Xas}{X_\acts}
\newcommand{\Xe}{X_\equi}

\newcommand{\wmin}{w^*}
\newcommand{\solfn}{\calW^*}

\newcommand{\Null}{\text{Null}}
\newcommand{\Row}{\text{Row}}
\newcommand{\Span}{\text{Span}}

\newcommand{\Ki}{K_{\bi}}

\newcommand{\Ka}{K_\act}

\newcommand{\ri}{\rho_{\bi}}

\newcommand{\ra}{\rho_{\act}}
\newcommand{\rmin}{\rho^*}

\newcommand{\diag}{\text{diag}}
\newcommand{\nnz}{\text{nnz}}

\usepackage{tikz}
\usepackage{pgfplots}

\usepgfplotslibrary{fillbetween}
\usetikzlibrary{patterns}

\tikzset{
    font={\fontsize{12pt}{12}\selectfont},
}

\pgfplotsset{
    compat=1.5.1,
    primary/.style={color=black, style=solid, line width=1.5pt}, 
    secondary/.style={color=red, style=solid, line width=1.5pt}, 
}

\icmltitlerunning{Optimal Sets and Solution Paths of ReLU Networks}

\begin{document}

\twocolumn[
	\icmltitle{Optimal Sets and Solution Paths of ReLU Networks}

	\icmlsetsymbol{equal}{*}

	\begin{icmlauthorlist}
		\icmlauthor{Aaron Mishkin}{cs_stanford}
		\icmlauthor{Mert Pilanci}{ee_stanford}
	\end{icmlauthorlist}

	\icmlaffiliation{cs_stanford}{Department of Computer Science, Stanford University}
	\icmlaffiliation{ee_stanford}{Department of Electrical Engineering, Stanford University}

	\icmlcorrespondingauthor{Aaron Mishkin}{amishkin@cs.stanford.edu}

	\icmlkeywords{Machine Learning, ICML}

	\vskip 0.3in
]



\printAffiliationsAndNotice{} 

\begin{abstract}
	We develop an analytical framework to characterize the set of optimal ReLU
	neural networks by reformulating the non-convex training problem as a convex
	program.
	We show that the global optima of the convex parameterization are given by
	a polyhedral set and then extend this characterization to the optimal set of the
	non-convex training objective.
	Since all stationary points of the ReLU training problem can be represented
	as optima of sub-sampled convex programs, our work provides a general expression
	for all critical points of the non-convex objective.
	We then leverage our results to provide an optimal pruning
	algorithm for computing minimal networks,
	establish conditions for the regularization path of ReLU networks
	to be continuous, and develop sensitivity results for
	minimal ReLU networks.
\end{abstract}

\section{Introduction}\label{sec:intro}


Neural networks have transformed machine learning.
Despite their success, little is known about
the global optima for typical non-convex training problems,
the solution path of regularized networks, or how to prune networks
without degrading the model fit.
This is in stark contrast to generalized linear models with \( \ell_2 \)
or \( \ell_1 \) penalties;
for example, it is well-known that the lasso \citep{tibshirani1996regression}
has a piece-wise linear path~\citep{osborne2000new, efron2004least},
a polyhedral solution set~\citep{tibshirani2013unique}, and
admits efficient algorithms for computing
minimal solutions~\citep{tibshirani2013unique}.
In this paper, we close the gap by studying neural networks
through the lens of convex reformulations.

\begin{figure}[t]
	\centering
	\includegraphics[width=0.5\textwidth]{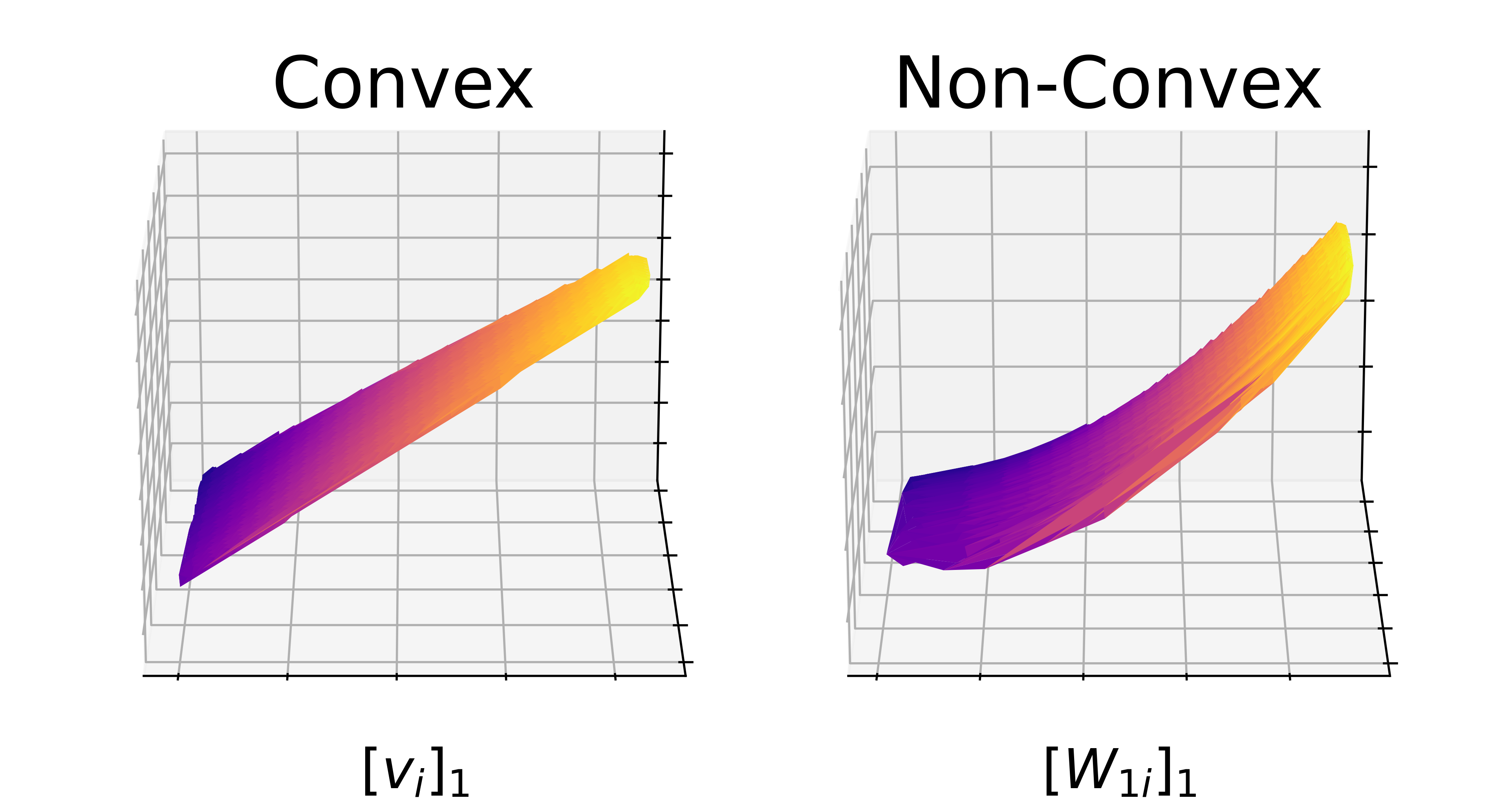}
	\vspace{-3ex}
	\caption{Convex vs non-convex solution spaces for two-layer ReLU
		networks. We plot the first feature
		of three different neurons;
		the non-convex parameterization maps the compact polytope of
		solutions for the convex parameterization into a curved manifold.
	}
	\label{fig:solution-set}
\end{figure}

One of the main challenges of neural
networks is non-convexity.
For non-convex problems, stationarity of the training objective does not
imply optimality of the network weights and so, to the best of our knowledge,
no work has been able to derive an analytical expression for the optimal set.
Convex reformulations provide a solution by rewriting the non-convex
optimization problem as a convex program in a lifted parameter space
\citep{pilanci2020convexnn}.
We focus on the convex reformulation for two-layer networks with ReLU activation
functions and weight decay regularization.
The resulting problem is related to the group lasso \citep{yuan2006model}
and induces \emph{neuron sparsity} in the network.

Let \( Z \!\in\! \R^{n \times d} \) be a data matrix and \( y \!\in\! \R^n \) associated targets.
The prediction function for two-layer ReLU networks is
\begin{equation*}
	f_{W_1, w_2}(Z) = \sum_{i=1}^m \rbr{Z W_{1i}}_+ w_{2i},
\end{equation*}
where \( W_1 \in \R^{m \times d} \), \( w_2 \in \R^m \) are the weights of
the first and second layers, \( m \) is the number of hidden units,
and \( \rbr{\cdot}_+ = \max\cbr{\cdot, 0} \) is the ReLU activation.
Fitting \( f_{W_1, w_2} \) with convex loss \( L \)
and weight decay (\( \ell_2 \)) regularization leads to the
standard non-convex optimization problem:
\begin{equation}\label{eq:non-convex-relu-mlp}
	\min_{W_1, w_2} \! L\Big(f_{W_1, w_2}(Z), y \Big) + \frac{\lambda}{2} \underbrace{(\norm{W_{1}}_F^2 + \norm{w_{2}}_2^2)}_{R(W_1, w_2)}.\!
\end{equation}
The regularization path or \emph{solution function} of this training problem
is the mapping between the regularization parameter \( \lambda \) and the set of
optimal model weights,
\begin{equation}\label{eq:solution-function}
	\calO^*(\lambda) = \argmin_{W_1, w_2}  L\Big(f_{W_1, w_2}(Z), y \Big)
	+ \frac{\lambda}{2} R(W_1, w_2).\!
\end{equation}
In general, the optimal neural network is not unique and \( \calO^*(\lambda) \)
will be set valued.
Indeed, there are always at least \( m! \) solutions since any permutation of
the hidden units yields an identical model.
We call the solution to a ReLU training \emph{p-unique} when it is unique
up to splits and permutations.

We study \( \calO^* \) by re-writing
\cref{eq:non-convex-relu-mlp} as an instance of the
\emph{constrained group lasso} (CGL).
We make the following contributions by analyzing CGL:\

\begin{itemize}
	\item We derive an analytical expression for the solution set of two-layer
	      ReLU networks and criteria for solutions to be p-unique (unique up to
	      permutations/splits).
	      \vspace{-0.5ex}

	\item We extend this characterization to show that the set of
	      stationary points of a two-layer ReLU model is exactly
	      \vspace{-0.5ex}
	      \begin{equation}\label{eq:stationary-points}
		      \begin{aligned}
			      \calC_\lambda =
			       & \big\{
			      (W_1, \! w_2) : \tilde \calD \! \subseteq \! \calD_Z,
			      \, f_{W_1, w_2}(Z) \! = \! \hat y_{\tilde \calD},                              \\
			       & W_{1i} = (\sfrac{\alpha_{i}}{\lambda})^{\sfrac{1}{2}} v_i(\tilde \calD), \,
			      w_{2i} = \xi_i(\alpha_i \lambda)^{\sfrac{1}{2}},                               \\
			       & \alpha_i \geq 0, \, i \in [m] \setminus \calS_\lambda \implies \alpha_i = 0
			      \big\},
		      \end{aligned}
	      \end{equation}
	      where \( \tilde \calD \) is a set of sub-sampled activation patterns,
	      \( \hat y_{\tilde \calD} \) is the unique optimal model fit using
	      those patterns, and \( v_i(\tilde \calD) \) are uniquely given by optimal
	      parameters for the dual of the convex reformulation.
	      See \cref{fig:solution-set}.
	      \vspace{-0.5ex}

	\item We provide an optimal pruning algorithm that can be used to compute
	      minimal models --- the smallest-width neural networks which are optimal
	      for a given dataset and regularization parameter --- and an intuitive
	      extension for pruning beyond minimal models.
	      \vspace{-0.5ex}

	\item We prove that the regularization path of ReLU networks is discontinuous
	      in general and establish sufficient conditions for path to be
	      closed/continuous.
	      \vspace{-0.5ex}

	\item We give a simple algorithm for computing the unique ReLU network
	      corresponding to the min-norm model in the convex lifting and,
	      under additional constraint qualifications, develop differential
	      sensitivity
	      results for minimal ReLU networks.
	      \vspace{-0.5ex}
\end{itemize}
In many cases, we obtain strictly stronger results for gated ReLU networks
\citep{fiat2019decoupling},
which correspond directly to an unconstrained group lasso problem \citep{mishkin2022convex}.
In particular, we give new sufficient conditions for (i) the group lasso
to be unique, (ii) global continuity of the group lasso model fit, and (iii)
weak differentiability of the solution function for gated ReLU networks.

The paper is structured as follows: we cover related work in
\cref{sec:related-work}
and introduce notation in \cref{sec:notation}.
Then we provide details for convex reformulations of neural network
in \cref{sec:convex-reformulations}.
\cref{sec:cgl} analyzes CGL and
\cref{sec:specialization} interprets these results in the specific context of
two-layer ReLU networks.
\cref{sec:experiments} concludes with experiments.

\subsection{Related Work}\label{sec:related-work}


\textbf{The Lasso and Group Lasso}:
Our work is most similar to \citet{hastie2007forward},
who consider homotopy methods, and
\citet{tibshirani2013unique}, who characterize the lasso solution set.
Limited results exist beyond the lasso.
\citet{tibshirani2011solution} analyze the generalized lasso,
while \citet{yuan2006model} show the group lasso is piece-wise linear
when \( X \) is orthogonal.
\citet{roth2008group} partially characterize the group lasso solution set,
while \citet{vaiter2012dof} derive stability results and the degrees-of-freedom.

\textbf{Convex Reformulations}:
Convex reformulations for neural networks have rapidly advanced
since \citet{pilanci2020convexnn};
convolutions \citep{ergen2021implicit, gupta2021exact},
vector-outputs \citep{sahiner2021vector}, batch-norm
\citep{ergen2021demystifying}, and deeper networks \citep{ergen2021beyond}
have all been explored.

\textbf{Neural Network Solution Sets}:
Characterizations of solution sets are largely empirical.
Mode connectivity has been studied extensively,
\citet{garipov2018mode, draxler2018barriers}.
\citet{nguyen2019connected, kuditipudi2019connectivity}
attempt to theoretically explain mode connectivity.
Sensitivity is connected to differentiable optimization
layers \citep{agrawal2019differentiable} and hypergradient descent
\citep{baydin2017online}.
We refer to \citet{bialock2020pruning} for an overview on pruning.

\subsection{Notation}\label{sec:notation}


We use lower-case \( a \) to denote vectors and upper-case \( A \) for matrices.
For \( d \in \bbN \), \( [d] = \cbr{1, \ldots, d} \).
Calligraphic letters \( \calC \) denote sets.
For a block of indices \( \bi \subseteq [d] \),
we write \( A_{\bi} \) for the sub-matrix of columns indexed by \( \bi \).
Similarly, \( a_{\bi} \) is the sub-vector indexed by \( \bi \).
If \( \calM \) is a collection of blocks, then \( A_\calM \) is the submatrix
and \( a_\calM \) the sub-vector with columns/elements indexed by blocks in
the collection.
Finally, \( \abs{\calM} \) is cardinality of the union of blocks in \( \calM \).

\section{Convex Reformulations}\label{sec:convex-reformulations}

Now we introduce background on convex reformulations.
Convex reformulations re-write \cref{eq:non-convex-relu-mlp} as a convex
program by enumerating the activations a single neuron in the
hidden layer can take on for fixed \( Z \) as follows:
\[ \calD_Z = \cbr{D = \diag(\mathbbm{1}(Z u \geq 0)) : u \in \R^d}. \]
This set grows as \(|\calD_Z| \le O(r(n/r)^r)\), where \(r := \text{rank}(Z)\)
\citep{pilanci2020convexnn}.
Each ``activation pattern'' \( D_i \in \calD_{Z} \) is associated with
a convex cone,
\[
	\calK_i = \cbr{u \in \R^d : (2D_i - I) Z u \succeq 0}.
\]
If \( u \in \calK_i \), then \( u \) matches \( D_i \), meaning \( D_i Z u = \rbr{Z u}_+ \).
For any subset \( \tilde \calD \subseteq \calD_Z \),
the convex reformulation is,
\begin{equation}\label{eq:convex-relu-mlp}
	\begin{aligned}
		\!\!\! \min_{v, w} \, & L\Big(\!\!\sum_{D_i \in \tilde \calD}\!\! D_i Z (v_i \!-\! u_i), y\Big)\!+\!\lambda\!\sum_{D_i \in \tilde \calD}\!\norm{v_i}_2\!+\!\norm{u_i}_2 \\
		                      & \text{s.t.} \quad v_i, u_i \in \calK_i.
	\end{aligned}
\end{equation}
\citet{pilanci2020convexnn} prove that this program and \cref{eq:non-convex-relu-mlp}
are equivalent in the following sense:
if \( \tilde \calD = \calD_Z \) and \( m \geq m^* \) for some \( m^* \leq n + 1 \),
then the two programs have the same optimal value and every solution to the convex program
can be mapped to a solution of the non-convex training problem and vice-versa.
Given a solution \( \rbr{v^*, u^*} \), optimal weights
for the ReLU problem are given by
\begin{equation}\label{eq:convex-to-relu}
	\begin{aligned}
		W_{1i} & = v_i^*/ \sqrt{\norm{v_i^*}}, \quad w_{2i} = \sqrt{\norm{v_i^*}}
		\\
		W_{1j} & = u_i^*/ \sqrt{\norm{u_i^*}}, \quad w_{2j} = -\sqrt{\norm{u_i^*}},
	\end{aligned}
\end{equation}
where we use the convention that \( 0/0 = 0 \).

In practice, learning with \( \calD_{Z} \) is intractable except when the data
are low rank.
\citet{mishkin2022convex} provide refined conditions on \( \tilde \calD \)
which are sufficient for \cref{eq:convex-relu-mlp} to be equivalent to the
non-convex problem, while \citet{wang2021hidden} show that
the minimum of every sub-sampled convex program is a stationary point of the
ReLU training problem.

\subsection{Gated ReLU Networks}

An alternative is the gated ReLU activation function,
\[
	\phi_{g}(Z, u) = \text{diag}(\mathbbm{1}(Z g \geq 0)) Z u,
\]
where \( g \in \R^d \) is a ``gate'' vector, which is also optimized.
The gated ReLU activation modifies the ReLU activation to
decouple the thresholding operator from the neuron weights.
Two-layer gated ReLU networks predict as follows:
\begin{equation}\label{eq:non-convex-gated-relu-mlp}
	h_{W_1, w_2}(Z) = \sum_{i=1}^m \phi_{g_i}(Z, W_{1i}) w_{2i}.
\end{equation}
\citet{mishkin2022convex} show that this gated ReLU neural network has
the convex reformulation,
\begin{equation}\label{eq:convex-gated-relu-mlp}
	\min_{u} \; L\Big(\sum_{D_i \in \tilde \calD} D_i Z w_i, y\Big)
	+ \lambda \sum_{D_i \in \tilde \calD} \norm{w_i}_2,
\end{equation}
where decoupling the activations from the neuron weights
allows \( u_i, v_i \in \calK_i \) to be merged.
The solution mapping for \( w^* \) and conditions for
for the convex program to be equivalent to
\cref{eq:non-convex-gated-relu-mlp} are similar to the ReLU case.


\section{The Constrained Group Lasso}\label{sec:cgl}


In this section, we develop properties of CGL,
a generalized linear model which captures both the convex ReLU and convex gated
ReLU programs.
Let \( \calB = \cbr{b_1, \ldots, b_m} \) be a disjoint partition of the
feature indices \( [d] \).
Given regularization parameter \( \lambda \geq 0 \),
CGL solves the program:
\begin{equation}\label{eq:constrained-gl}
	\begin{aligned}
		p^*(\lambda) \!=\! \min_{w} \,
		 & F_\lambda(w) \!:=\! \half \norm{X w \!-\! y}_2^2 + \lambda \sum_{\bi \in \calB} \norm{\wi}_2 \\
		 & \quad \text{s.t.} \quad \Ki^\top \wi \leq 0  \text{ for all } \bi \in \calB,
	\end{aligned}
\end{equation}
where \( \Ki \in \R^{|\bi| \times a_\bi} \).
Solutions to \cref{eq:constrained-gl} are block sparse when \( \lambda \)
is sufficiently large, meaning \( \wi = 0 \) for a subset of blocks.
This is similar to the feature sparsity given by the lasso,
to which CGL naturally reduces when \( \bi = \cbr{i} \)
and \( \Ki = 0 \) for each \( \bi \in \calB \).
Although we consider squared-error, our results generalize to
strictly convex losses --- see \cref{app:general-losses} for comments.

The convex reformulations introduced in the previous section are instances
of CGL using the basis function \( X = [D_1 Z \ldots D_p Z] \),
where \( p = |\tilde \calD_Z| \).
For gated ReLU models, \( \Ki = 0 \) while ReLU models set
\( \Ki = - Z^\top (2 D_i - I) \).
For both problems, block sparsity from the group \( \ell_1 \) penalty
induces neuron sparsity in the final solution.

Our goal is to characterize the solution function of CGL,
\[
	\solfn(\lambda) := \argmin_{w \, : \, K_\bi^\top \wi \leq 0} F_\lambda(w)
\]
For a general data matrix, \( F_\lambda \) is not strictly convex and
CGL may admit multiple solutions --- these correspond to networks which
are not related by permutation.
As such, \( \solfn \) is a point-to-set map and
we must use a criterion to define a function;
for instance, the min-norm solution mapping
\[
	\wmin(\lambda) = \argmin \cbr{ \norm{w}_2 : w \in \solfn(\lambda) },
\]
defines a function for all \( \lambda \geq 0 \).

Now we introduce notation that will be used throughout this section.
Let \( \hat y(\lambda) = X w \) for \( w \in \solfn(\lambda) \) denote
the optimal model fit, which is the same for any choice of optimal \( w \)
(\cref{lemma:unique-fit-cgl}).
Similarly, define the optimal residual \( r(\lambda) := y - \hat y(\lambda) \)
and \( \ci(\lambda) := \Xbi^\top r(\lambda) \) as the
correlation vector for block \( \bi \).
We write \( c \in \R^d \) for the concatenation
of these block-vectors.
Finally, let \( \ri \) be the dual parameters for the constraint
\( \Ki^\top \wi \leq 0 \), \( \rho \) their concatenation,
and \( K \) the block-diagonal matrix with blocks given by \( \Ki \).

The Lagrangian associated with \cref{eq:constrained-gl} is
\begin{equation}\label{eq:lagrangian-cgl}
	\calL(w, \rho) = \half \norm{X w - y}_2^2 + \lambda \sum_{\bi \in \calB} \norm{\wi}_2 + \abr{K \rho, w}.
\end{equation}
The constraints are linear, strong duality attains if feasibility holds,
and the necessary and sufficient conditions for primal-dual pair \( (w, \rho) \)
to be optimal are the KKT conditions:
\begin{equation}\label{eq:kkt-cgl}
	\begin{aligned}
		\Xbi^\top (X w - y) + \Ki \ri + s_\bi & = 0                                \\
		\Ki^\top \wi                          & \leq 0                             \\
		[\ri]_j \cdot [\Ki]_j^\top \wi        & = 0 \quad \forall \, j \in [a_\bi] \\
		\ri                                   & \geq 0,
	\end{aligned}
\end{equation}
where \( s_\bi \in \partial \lambda \norm{\wi}_2 \).
Since the KKT conditions hold for every combination of optimal primal-dual
pair \citep{bv2014convex}, we always use the min-norm dual
optimal parameter \( \rmin \) with no loss of generality.
To simplify our notations, we define \( \vi := \ci - \Ki \rmin_\bi \).
In what follows, all proofs are deferred to \cref{app:cgl}.

\subsection{Describing the Optimal Set}

Stationary of the Lagrangian implies the equicorrelation set
\[
	\equi = \cbr{ \bi \in \calB : \norm{\vi}_2 = \lambda }.
\]
\vspace{-0.5ex}%
contains all blocks which may be active for fixed \( \lambda \).
That is, the active set \( \act(w) = \cbr{\bi : \wi \neq 0} \) satisfies
\( \act(w) \subseteq \equi \) for every \( w \in \solfn(\lambda) \).
However, not all blocks in \( \equi \) may be non-zero for some solution.
Thus, we define
\[
	\calS_\lambda = \cbr{\bi \in \calB : \exists w \in \solfn(\lambda), \wi \neq 0},
\]
which is the set of blocks supported by some solution.

Our first result combines KKT conditions with uniqueness of \( \hat y(\lambda) \)
to characterize the solution set for fixed \( \lambda > 0 \).
\begin{restatable}{proposition}{solFnCGL}\label{prop:sol-fn-cgl}
	Fix \( \lambda > 0 \).
	The optimal set for the CGL problem is
	given by
	\begin{equation}\label{eq:sol-fn-cgl}
		\begin{aligned}
			\solfn(\lambda) \!=\!
			\big\{ & w \! \in \! \R^d \!:\!
			\forall \, \bi \! \in \! \calS_\lambda,
			\wi \!= \! \alpha_\bi \vi, \alpha_\bi \!\geq\! 0, \,             \\
			       & \quad \forall \, b_j \in \calB \setminus \calS_\lambda,
			\w_{b_j} = 0, \, X w = \hat y
			\big\}
		\end{aligned}
	\end{equation}
\end{restatable}

\vspace{-1.5ex}%
Since this characterization is implicit due to the dependence
on \( \calS_\lambda \), we also give an alternative and explicit
construction in \cref{prop:cgl-alternate}, which shows that when \( K = 0 \)
we may replace \( \calS_\lambda \)
with \( \equi \);
We prefer \cref{prop:sol-fn-cgl} to \cref{prop:cgl-alternate} since it
better mirrors this simpler setting.
However, \cref{prop:cgl-alternate} can be substituted wherever desired.

Now that we know ``shape'' of the solution set, it is
possible to obtain simple conditions for existence of a unique solution.
As an immediate consequence of \cref{prop:sol-fn-cgl},
the solution map is a subset of directions in \( \Null(\Xe) \).
\begin{corollary}\label{cor:uniqueness-cond-cgl}
	If \( w, w' \in \solfn(\lambda) \) and \( z' = w - w' \), then
	\[
		z_\equi' \in \calN_\lambda \! := \! \Null(\Xe) \cap \cbr{z_\equi :  \forall \bi \in \equi, \, \zi \! = \! \alpha_\bi \vi}.
	\]
	As a result, the group lasso solution is unique if \( \calN_\lambda = \cbr{0} \).
\end{corollary}

\cref{cor:uniqueness-cond-cgl} extends a similar result for the lasso
to CGL \citep[Eq. 9]{tibshirani2013unique} and
implies the solution is unique for all \( \lambda \geq 0 \) if
the columns of \( X \) are linearly independent.
The corollary also provides a simple check
for primal uniqueness given a primal-dual solution pair.

\begin{restatable}{lemma}{uniqueCGL}\label{lemma:unique-cgl}
	Fix \( \lambda > 0 \).
	The solution to CGL problem is unique
	if and only if \( \cbr{\Xbi \vi}_{\calS_\lambda} \)
	are linearly independent.
\end{restatable}

Note that a dual solution \( \rho \) is necessary to compute \( v \) in general;
By uniformizing over \( \vi \), we obtain a stronger condition that can be checked
whenever \( \equi \) is known, yet is still
weaker than linear independence of the columns
of \( X \).

\begin{corollary}
	If the columns of \( X_\equi \) are linearly independent,
	then the CGL problem has a unique solution.
\end{corollary}

Finally, we consider the special case when there are no constraints and
\( K = 0 \).
In this setting, \( \vi = \ci \) --- the dual parameters are trivially zero
--- and we can provide a global condition which is much stronger than linear
independence.

\begin{restatable}{proposition}{GGP}[Group General Position]\label{prop:unique-sol}
	Suppose for every \( \calE \subseteq \calB \), \( \abs{\calE} \leq n +1 \),
	there do not exist unit vectors \( \zi \in \R^{\abs{\bi}} \)
	such that for any \( j \in \calE \),
	\[
		X_{b_j} z_{b_j} \in
		\text{affine}(\cbr{\Xbi \zi : \bi \in \calE \setminus b_j}).
	\]
	Then the group lasso solution is unique for all \( \lambda > 0 \).
\end{restatable}
We call this uniqueness condition \emph{group general position}
(GGP) because it naturally extends general position to groups of vectors.
General position is itself an extension of affine independence and is
sufficient for the lasso solution to be unique \citep{tibshirani2013unique}.
GGP is strictly weaker than linear independence of the columns of \( X \),
but neither implies nor is implied by general position (\cref{prop:general-position-relation}).

\subsection{Computing Dual Optimal Parameters}

The main difficulty of \cref{lemma:unique-cgl}
is that knowledge of a dual optimal parameter is required to check if a
unique solution exists.
A dual optimal parameter is also required to fully leverage our
characterization of the optimal set.
As such, now we turn to computing optimal dual parameters.

We give one Lagrange dual problem for CGL in \cref{lemma:lagrange-dual}.
A nice feature of this dual problem is \( \vi \) attains an
alternative interpretation as dual variable.
However, evaluating the dual requires computing \( (X^\top X)^+ \),
which may be difficult even if \( X \) is structured, as in the convex
ReLU program.
Instead, we focus on computing \( \rho \) given a primal solution.

Let \( \w \in \solfn(\lambda) \).
If \( \wi \neq 0 \), then KKT conditions imply
\begin{equation}\label{eq:dual-active}
	\Ki \ri = \ci - \lambda \frac{\wi}{\norm{\wi}_2},
\end{equation}
so that the ``dual fit'' \( \hat d_\bi = \Ki \ri \) is easily computed.
Recovering the dual parameter is a linear feasibility problem:
\begin{equation}\label{eq:dual-active-lp}
	\ri \in \cbr{\ri \geq 0 : \Ki \ri = \hat d_\bi}.
\end{equation}
If \( \wi = 0 \), then complementary slackness is trivially satisfied and we
compute the min-norm dual parameter by solving the following program:
\begin{equation}\label{eq:dual-inactive}
	\min \cbr{\norm{\ri}_2 : \norm{\ci - \Ki \ri}_2 \leq \lambda, \ri \geq 0 }.
\end{equation}
In general, however, we only need \emph{some} dual optimal parameter for
our results to hold; thus, is is typically easier to find \( \rho \)
by solving the following non-negative regression:
\begin{equation}\label{eq:non-negative-regression}
	\ri = \argmin \cbr{\norm{\Ki \ri - \ci}^2_2 : \ri \geq 0 }.
\end{equation}
See \cref{prop:non-negative-regression} for details.

\subsection{Minimal Solutions and Optimal Pruning}\label{sec:pruning}

Often we want the most parsimonious solution, i.e. the one
using the fewest feature groups.
We say a primal solution \( w \) is \emph{minimal} if there
does not exist \( w' \in \solfn(\lambda) \) such that
\( \act(w') \subsetneq \act(w) \).
Building on the previous section, we start with a sufficient
condition for \( w \) to be minimal.
\begin{restatable}{proposition}{minimalSols}\label{prop:minimal-solutions}
	For \( \lambda > 0 \), \( w \in \solfn(\lambda) \) is minimal if and only
	if the vectors \( \cbr{\Xbi \wi}_{\calA(w)} \) are linearly independent.
\end{restatable}

Linear independence of \( \cbr{\Xbi \wi}_{\calA(w)} \)
also identifies \( w \) as a vertex of \( \solfn(\lambda) \)
\citep{bertsekas2009convex},
meaning minimal models are exactly the extreme points of the optimal set.
Combining this characterization with our condition
for uniqueness of a solution (\cref{lemma:unique-cgl}) shows that
minimal solutions are the only solution on their support.

\begin{corollary}\label{cor:minimal-unique}
	Suppose \( w \) is a minimal solution.
	Then \( w \) is the unique solution with support \( \act(w) \).
\end{corollary}

If \( X \) satisfies \emph{group dependence}
(\cref{def:group-dependence}), then all minimal solutions have the same
number of active blocks.

\begin{restatable}{proposition}{minimalSolCharacterization}\label{prop:minimal-sol-characterization}
	Let \( \calV \!=\! \Span(\cbr{\Xbi \bar \wi}) \) for \( \bar w \in \solfn(\lambda) \).
	If \( X \) satisfies group dependence, then every minimal solution has
	\( c = \text{dim}(\calV) \) active blocks.
\end{restatable}

\begin{algorithm}[tb]
	\caption{Optimal Solution Pruning}
	\label{alg:pruning-solutions}
	\begin{algorithmic}
		\STATE {\bfseries Input:} data matrix \( X \), solution \( w \).
		\STATE \( k \gets 0 \).
		\STATE \( w^k \gets w \).
		\WHILE {\( \exists \beta \neq 0 \) s.t. \( \sum_{\bi \in \act(w^k)} \beta_\bi \Xbi \wi^k = 0 \)}
		\STATE \( \bi^k \gets \argmax_{\bi} \cbr{|\beta_\bi| : \bi \in \act(w^k)}  \)
		\STATE \( t^k \gets 1/|\beta_{\bi^k}| \)
		\STATE \( \w^{k+1} \gets \w^k (1 - t^k \beta_\bi) \)
		\STATE \( k \gets k + 1 \)
		\ENDWHILE
		\STATE {\bfseries Output:} final weights \( \w^k \)
	\end{algorithmic}
\end{algorithm}

\cref{alg:pruning-solutions} gives a procedure which, starting from
any optimal solution \( w \), computes a optimal model with minimal
support in \( O((n^3 l + nd) \) time, where \( l \)
is the number active blocks in \( w \)
(see \cref{prop:pruning-correctness}).
Moreover, if the blocks of \( X \) satisfy a regularity condition, then
starting model has no affect on the cardinality of the minimal solution found.
Our algorithm can also be used to verify a minimal solution, since
if \( w \) minimal then it is unique on its support and
\cref{alg:pruning-solutions} must return \( w \) immediately.
This procedure also implies the existence of at least one minimal
solution.

\begin{corollary}\label{cor:smallest-solution-cgl}
	There exists \( w \in \solfn(\lambda) \) for which the vectors
	\(
	\cbr{\Xbi w(\lambda) : \bi \in \calA(w) }
	\)
	are linearly independent.
\end{corollary}

\cref{cor:smallest-solution-cgl} will be useful tool later when we study
sensitivity of the model fit to perturbations in \( y \) and \( \lambda \).

A disadvantage of \cref{alg:pruning-solutions} is that it cannot
continue beyond a minimal solution.
However, minimal models may still be quite large.
We can perform approximate pruning in such cases
using the least squares fit to approximate \( \beta \),
\[ \tilde \beta = \argmin \norm{A \beta - X_{b_j} w_{b_j}}_2^2, \]
where \( A = [\Xbi \wi]_{\act \setminus b_j} \) and \( b_j \in \act \) is
chosen randomly.
Using \( \tilde \beta \) in
\cref{alg:pruning-solutions}
is optimal when \( \cbr{\Xbi \wi}_{\act} \) are dependent
and chooses the update parameters to minimize degradation
of the model fit otherwise.

\subsection{Continuity of the Solution Path}

A major concern when learning with regularizers is how to tune the
parameter \( \lambda \).
Typical strategies like grid-search on the (cross) validation
loss are effective only if the solution function satisfies basic continuity
properties.
For example, if \( \solfn \) is single-valued but discontinuous in \( \lambda \),
then the sample complexity of grid-search can be made arbitrarily poor
by ``hiding'' the optimal \( \lambda \) in a
discontinuity \citep[Sec. 1.1]{nesterov2018lectures}.
In this section, we justify grid-search for CGL by proving several
continuity properties of the solution function, particularly when the
solution is unique.
We start with basic definitions of continuity for point-to-set maps.

\begin{definition}[Closed]
	\( T : \calX \into 2^\calZ \) is closed if
	\( \cbr{\xk} \subset \calX \), \( \xk \into \bar x \) and \( z_k \in T(\xk) \),
	\( z_k \into \bar z \) implies \( \bar z \in T(\bar x) \).
\end{definition}

\begin{definition}[Open]
	\( T : \calX \into 2^\calZ \) is open if
	\( \cbr{\xk} \subset \calX \), \( \xk \into \bar x \) and \( \bar z \in T(\bar x) \),
	implies there exists \( k' \in \bbN \), \( z_k \in T(\xk) \) for \( k \geq k' \),
	such that \( z_k \into \bar z \).
\end{definition}

We say that \( T \) is \emph{continuous} if it is both closed and open.
If \( T(x) \) is a singleton for all \( x \in \calX \),
then openness/closedness are equivalent and imply continuity.
We start with (functional) continuity of the optimal objective.

\begin{restatable}{proposition}{valueContinuityCGL}\label{prop:value-continuity-cgl}
	\( \lambda \mapsto p^*(\lambda) \) is continuous for all
	\( \lambda \geq 0 \).
\end{restatable}

While standard sensitivity results imply that \( \solfn \) is closed,
unfortunately openness is not possible in the general setting.

\begin{restatable}{proposition}{mapContinuityCGL}\label{prop:map-continuity-cgl}
	While \( \solfn \) is closed on \( \R_+ \),
	it is open if only if \( X \) is full column rank.
	However, if the solution is unique on
	\( \Lambda \subset \R^+ \), then \( \solfn \)
	is open at every \( \lambda \in \Lambda \).
\end{restatable}

As a corollary of \cref{prop:map-continuity-cgl}, \( \solfn \) is open
on \( \R^+ \) if and only if \( X \) is full column rank.
Continuity of \( \solfn \) is impossible in general because,
as \citet{hogan1973point} shows, openness is a local stability property;
since \( \solfn(0) \) is unbounded, many ``unstable'' solutions
exist at \( \lambda = 0 \) which are not limit points of other solutions.
Continuity of the unique solution path is an immediate corollary of
\cref{prop:map-continuity-cgl}.

\begin{corollary}
	If the CGL solution is unique on an interval \( \Lambda \subset \R_+\),
	then it is also continuous on \( \Lambda \).
\end{corollary}

In particular, if \( K = 0 \) and GGP holds, then the group lasso solution
is continuous for all \( \lambda > 0 \).
We can strengthen our continuity results when \( \Ki = 0 \) in another way:
by analyzing the dual of the group lasso problem,
we extend continuity from \( p^* \) to the optimal model fit.

\begin{restatable}{proposition}{modelFitContinuity}\label{prop:model-fit-continuity}
	If \( K = 0 \),
	then \( \hat y(\lambda)  \) is continuous on \( \R_+ \) and the penalty
	\( \sum_{\bi \in \calB} \norm{\wi(\lambda)}_2 \) is continuous
	for \( \lambda > 0 \).
\end{restatable}

\subsection{The Min-Norm Path}\label{sec:min-norm}

Now we turn our attention to the min-norm solution path.
Min-norm solutions are typically used in under-determined problems since the
norm of the solution is connected to generalization
\citep{neyshabur2015bias, gunasekar2017implicit}.
Furthermore, the min-norm solution is
a function of \( \lambda \), unlike \( \solfn \).
Throughout this section, \( \acts = \act(\wmin) \) denotes
the active set of the min-norm solution.

Unfortunately, studying the min-norm path immediately encounters a
surprising difficulty:
as opposed to least-squares problems or the lasso (see \citet{tibshirani2013unique}),
the min-norm solution may not lie in the row space of the active set.

\begin{restatable}{proposition}{rowCE}\label{prop:row-ce}
	Suppose \( \Ki = 0 \).
	There exists \( (X, y) \) and
	\( \lambda > 0 \) such that
	\( \wmin_\acts(\lambda) \not \in \Row(\Xas) \).
\end{restatable}

Since the min-norm solution is not given by projecting onto \( \Row(\Xas) \),
how can we compute and study it?
Again, our characterization for the solution set provides a way forward.

\begin{restatable}{proposition}{minNormProgram}\label{prop:min-norm-program}
	Let \( \lambda > 0 \)
	and consider the program:
	\begin{equation}\label{eq:min-norm-program}
		\begin{aligned}
			\alpha^* = \argmin_{\alpha \geq 0} \norm{\alpha}_2^2
			\, \, \text{ s.t.} \, \,
			\sum_{\bi \in \calS_\lambda} \alpha_{\bi} \Xbi \vi = \hat y.
		\end{aligned}
	\end{equation}
	Then the min-norm solution is given by
	\( \wmin_\bi = \alpha^*_\bi \vi \).
\end{restatable}

\cref{eq:min-norm-program} is a quadratic program (QP) that can be solved with
off-the-shelf software like \textsc{cvxpy} \citep{diamond2016cvxpy}.
If \( |\calB| \) and \( d \) are large, this QP may be too expensive to handle directly.
In such situations, we propose to the solve the following elastic-net-type problem
\begin{equation}\label{eq:l2-penalized-cgl}
	\begin{aligned}
		\min_{w} \,
		 & \half \norm{X w - y}_2^2 + \lambda \sum_{\bi \in \calB} \norm{\wi}_2
		+ \frac{\delta}{2} \norm{w}_2^2                                                 \\
		 & \quad \text{s.t.} \quad \Ki^\top \wi \leq 0  \text{ for all } \bi \in \calB.
	\end{aligned}
\end{equation}

This \( \ell_2 \)-penalized CGL problem is equivalent to
CGL with modified dataset \( (\tilde X, \tilde y) \) (\cref{lemma:l2-reformulation}).
Since the optimization problem is strongly convex (\( \tilde X \) is full column rank),
invoking \cref{prop:map-continuity-cgl} implies the solution \( w^\delta(\lambda) \) is continuous for all \( \lambda \geq 0 \).
Moreover, as \( \delta \into 0 \), the penalized solution converges
to the min-norm solution to CGL.
\begin{restatable}{proposition}{penConvergence}\label{prop:l2-convergence}
	The solution to the \( \ell_2 \)-penalized problem converges to
	the min-norm solution as \( \delta \rightarrow 0 \).
	That is,
	\[
		\lim_{\delta \into 0} w^\delta(\lambda) = \wmin(\lambda).
	\]
\end{restatable}

\vspace{-2ex}%
Uniqueness and continuity of the solution path for the penalized CGL
problem mean we may prefer to solve
Problem~\eqref{eq:l2-penalized-cgl} with small \( \delta > 0 \) when
tuning \( \lambda \).
\cref{prop:l2-convergence} guarantees that the bias induced by \( \delta \)
will be small and a polishing step with \( \delta = 0 \) can always be used.
Finally, non-zero \( \delta \) ensures the objective is strongly convex,
meaning we can use linearly convergent methods to solve the problem.


\subsection{Sensitivity}

Now we move onto the problem of sensitivity of a solution
\( w \in \solfn(\lambda, y) \)
to perturbations, either in \( \lambda \) or the targets \( y \).
The main tool for measuring such perturbations are the gradients, for example
\( \nabla_\lambda w(\lambda, y) \).
However, since the solution path of the group lasso is non-smooth,
we must cope with the fact that gradients are not available everywhere.

We show that the gradients of minimal solutions exist
almost everywhere under additional constraint qualifications (CQs).
We do so by considering a \emph{reduced} problem and showing that the
solution to this reduced problem is exactly \( \wa \).
Define the reduced problem as follows:
\begin{equation}\label{eq:constraint-reduced}
	\begin{aligned}
		\min_{\wa} \,
		 & \half \norm{\Xa \wa - y}_2^2 + \lambda \sum_{\bi \in \act} \norm{\wi}_2 \\
		 & \quad \text{s.t.} \quad \Ka^\top \wa \leq 0
	\end{aligned}
\end{equation}
If \( \act(w) \) is the support of a minimal solution, then \( w \) is the only
solution with support \( \act \) and \cref{eq:constraint-reduced} can be used
to compute the unique active weights.
\begin{restatable}{proposition}{constraintReduced}\label{prop:constraint-reduced}
	Let \( w \in \solfn(\lambda, y) \) be minimal.
	The active blocks \( \wa \) are the unique solution to
	Problem~\eqref{eq:constraint-reduced}.
\end{restatable}
We use this fact to obtain a local solution function
for CGL using the implicit function theorem.
Given a solution \( w \), let
\[
	\calB(w) = \bigcup_{\bi \in \act} \cbr{j \in [a_\bi]: [\Ki]_j^\top \wi = 0 },
\]
be the active constraints.
We now need two classical CQs.

\begin{definition}[LICQ]
	\( w \! \in \! \solfn(\lambda, y) \) satisfies linear independence
	CQ if \( \cbr{[K]_j : j \! \in \! \calB(w)} \) are linearly independent.
\end{definition}

\begin{definition}[SCS]
	Primal solution \( w \in \solfn(\lambda) \) satisfies
	strict complementary slackness if there exists a dual optimal parameter
	\( \rho \) such that \( [\rho]_j > 0 \) for every \( j \in \calB \).
\end{definition}

Now we can state our main differential sensitivity result.

\begin{restatable}{proposition}{localSolFn}\label{prop:local-sol-fn}
	Let \( w \in \solfn(\bar \lambda, \bar y) \) be minimal and suppose
	\( w \) satisfies LICQ on the active set \( \act \)
	and SCS on the equicorrelation set \( \equi \).
	Then \( w \) has a locally continuous solution function
	\( (\lambda, y) \mapsto w(\lambda, y) \).
	Moreover, if
	\[
		D =
		\begin{bmatrix}
			\Xa^\top \Xa + M(\bar w) & \Ka                            \\
			\bar \ra \odot \Ka       & \text{diag}(\Ka^\top \bar \wa)
		\end{bmatrix},
	\]
	where \( \odot \) is the element-wise product,
	\( u_{b_i} = \frac{\wi}{\norm{\wi}_2} \),
	\( u \) is the concatenation of these vectors,
	and \( M \) is block-diagonal projection matrix in \cref{eq:projection-matrix-block},
	then the Jacobians of \( w(\bar \lambda, \bar y) \)
	with respect to \( \lambda \)
	and \( y \) are given as follows:
	\[
		\nabla_\lambda w(\bar \lambda, \bar y)
		= - [D^{-1}]_\act u_\act
		\hspace{0.5em}
		\nabla_y w(\bar \lambda, \bar y)
		= [D^{-1}]_\act \Xa^\top,
	\]
	where \( [D_\act^{-1}]_\act \) is the \( |\act| \times |\act| \)
	dimensional leading principle submatrix of \( D \).
\end{restatable}

\section{Specialization to Neural Networks}\label{sec:specialization}


\begin{algorithm}[tb]
	\caption{Approximate ReLU Pruning}
	\label{alg:pruning-solutions-nn}
	\begin{algorithmic}
		\STATE {\bfseries Input:} data matrix \( Z \), weights \( W_{1}, w_{2} \),
		score function \( s \).
		\STATE \( m \gets \abs{\act(W_1)} \)
		\STATE \( (W_1^0, w_2^0) \gets (W_{1}, w_2) \).
		\STATE \( q^0_i \gets (X W^0_{1i})_+ w^0_2 \)
		\FOR {\( k = 0 \text{ to } m -1 \)}
		\STATE \( j^k = \argmin_{i \in \act(W_{1i}^k)} s(W_{1i}^k) \)
		\STATE \( \beta^k = \argmin_{\beta} \norm{ \sum_{i \neq j^k} \beta_{i} q^k_i - q_{j^k}^k}_2^2 \)
		\STATE \( i^k \gets \argmax_{i} \cbr{|\beta_i| : i \in \act(W_{1i}^k)}  \)
		\STATE \( t^k \gets 1/|\beta_{i^k}| \)
		\STATE \( (W_{1i}^{k+1}, w_{2i}^{k+1}) \gets (W_{1i}^k, w_{2i}^k)
		\cdot (1 - t^k \beta_i)^{1/2} \)
		\STATE \( q^{k+1}_i \gets q_i^{k} \cdot (1 - t^k \beta_i) \)
		\ENDFOR
		\STATE {\bfseries Output:} final weights \( W_1^k, w_2^k \)
	\end{algorithmic}
\end{algorithm}

Now we specialize our results for CGL to two-layer neural networks
with ReLU or gated ReLU activations.
We state and prove our results for ReLU networks, but they are easily
adapted to gated ReLUs.
We start by interpreting conditions for uniqueness in the context
of non-convex ReLU models and then move on to discussing optimal pruning
for ReLU networks and continuity properties of the solution function.
Proofs are deferred to \cref{app:specializations}.

\textbf{Optimal Sets and Uniqueness}:
Combining the mapping between solutions for the convex reformulation
and the original non-convex training problem (\cref{eq:convex-to-relu})
and \cref{prop:sol-fn-cgl}
immediately allows us to characterize the solution set
for the full ReLU problem:
\begin{restatable}{theorem}{reluSolFn}\label{cor:relu-sol-fn}
	Suppose \( m \geq m^* \) and \( \tilde D = \calD_Z \) (no sub-sampling),
	with \( p = |\tilde \calD_Z| \).
	Then the optimal set for the ReLU problem up to permutation/splitting
	symmetries is
	\vspace{-1ex}
	\begin{equation}\label{eq:relu-sol-fn}
		\begin{aligned}
			\hspace{-0.5em} \calO_\lambda
			\!= & \,
			\big\{
			(W_1, \! w_2) \!:\!
			\, f_{W_1, w_2}(Z) \! = \! \hat y,
			W_{1i} \!=\! (\sfrac{\alpha_{i}}{\lambda})^{\sfrac{1}{2}} v_i, \\
			    & w_{2i} \!=\! \xi_i (\alpha_i \lambda)^{\sfrac{1}{2}},
			\alpha_i \!\geq\! 0, \, i \!\in\! [2p] \!\setminus\! \calS_\lambda \!\Rightarrow\! \alpha_i \!=\! 0
			\big\}.
		\end{aligned}
	\end{equation}
\end{restatable}
\vspace{-1ex}%
Eq.~\eqref{eq:relu-sol-fn} abuses notation by using $\calS_\lambda$ as
a subset of the neuron indices \( \cbr{1, \ldots 2p} \), where
\( \cbr{1, \ldots p} \) index positive neurons for which \( \xi_i = +1 \)
(corresponding to blocks \( D_i Z \))
and \( \cbr{p+1, \ldots, 2p} \) index the negative neurons (corresponding to \(
- D_i Z \)), for which \( \xi_i = -1 \).
\cref{fig:solution-set} plots the first feature of three neurons
as they vary over this solution set.
The mapping from convex to non-convex parameterization
transforms the polytope of solutions into a curved manifold.

Choosing a sub-sampled set of patterns
\( \tilde \calD \subset \calD_Z \) corresponds to finding a stationary
point of the non-convex training problem \citep{wang2021hidden}.
Using this fact with \cref{cor:relu-sol-fn} finally justifies
description of all stationary points of the ReLU
problem given in the introduction.
\begin{restatable}{proposition}{stationaryPoints}\label{prop:stationary-points}
	The set of stationary points of two-layer ReLU networks up to
	permutation/splitting symmetries is
	\vspace{-0.5ex}%
	\begin{equation*}
		\begin{aligned}
			\calC_\lambda = \big\{
			 & (W_1, \! w_2) : \tilde \calD \! \subseteq \! \calD_Z,
			\, f_{W_1, w_2}(Z) \! = \! \hat y_{\tilde \calD},                                             \\
			 & W_{1i} = (\sfrac{\alpha_{i}}{\lambda})^{\sfrac{1}{2}} v_i(\tilde \calD), \,
			w_{2i} = \xi_i(\alpha_i \lambda)^{\sfrac{1}{2}},                                              \\
			 & \alpha_i \geq 0, \, i \in [2 |\tilde \calD|] \setminus \calS_\lambda \implies \alpha_i = 0
			\big\},
		\end{aligned}
	\end{equation*}
	where \( \tilde \calD \) are sub-sampled activation patterns,
	\( \hat y_{\tilde \calD} \) is the optimal model fit using
	those patterns, and
	\( v_i(\tilde \calD) = \ci(\tilde \calD) - \Ki \ri(\tilde \calD) \)
	is determined by the fit and the dual parameters.
\end{restatable}

We note that since deeper networks are also related
to CGL through convex reformulations \citep{ergen2021beyond},
\cref{prop:stationary-points} may also be applied beyond two layers.

Recall that the solution set for a two-layer ReLU network
is typically not unique due to model symmetries.
However, if the convex solution is unique, then
the non-convex ReLU training problem is p-unique.
Combining this with our results for CGL
gives the following sufficient conditions.

\begin{restatable}{proposition}{pUnique}\label{prop:p-unique}
	Let \( \lambda > 0 \)
	and suppose that the convex ReLU problem
	has a unique solution.
	Then the ReLU model solution is p-unique.
	In particular,
	if \( \cbr{D_i Z \vi}_\equi \) are linearly independent,
	then the non-convex solution is p-unique.
\end{restatable}

For gated ReLU networks, it is also sufficient to check the blocks
\( \sbr{D_i Z}_{D_i \in \tilde \calD} \) to see if they satisfy GGP.
By looking at the structure of \( D_i X \),
we give simple sufficient conditions for sub-sampled
convex ReLU programs to be p-unique.

\begin{restatable}{proposition}{reluUnique}\label{prop:relu-unique}
	Let \( \lambda > 0 \) and \( p = |\tilde \calD| \).
	Suppose \( Z \) follows a continuous probability distribution
	and \( \nnz(D_i) \geq p \cdot d \) for every \( D_i \in \tilde \calD \).
	If \( \equi \) does not contain two blocks with the same
	activation pattern, then
	the sub-sampled convex ReLU program has a p-unique solution
	almost surely.
\end{restatable}

\cref{prop:relu-unique} requires \( n \) to be much greater than
\( d \) to be useful due to the trivial bound \( \nnz(D_i) \leq n \).
In practice, the condition on activation patterns can be enforced by
constraining \( v_i = 0 \) or \( w_i = 0 \) for each
activation pattern before solving the convex reformulation.

\setlength{\tabcolsep}{2.5pt}
\begin{table}[t]
	\centering
	\caption{Tuning neural networks by searching over the optimal set.
		We fit two-layer ReLU networks on the training set and
		compute the minimum \( \ell_2 \) norm solution (Min L$_2$).
		Then we tune by finding an extreme point approximating
		the maximum \( \ell_2 \)-norm solution (EP), minimizing validation
		MSE over the optimal set (V-MSE), and minimizing test MSE over
		the optimal set (T-MSE).
		Results show median test accuracy;
		Max Diff. reports the difference between the best and worse models
		found.
		Exploring the optimal set reveals a huge disparity in the performance
		of optimal networks, with the generalization gap exceeding \( 20 \)
		points on four datasets.
	}
	\label{table:tuning-small}
	\vspace{0.1in}
	\begin{tabular}{lccccc} \toprule
		\textbf{Dataset} & \textbf{Min L$_2$} & \textbf{EP} & \textbf{V-MSE} & \textbf{T-MSE} & \textbf{Max Diff.} \\
		\midrule
		fertility        & 0.66               & 0.69        & 0.65           & 0.64           & 0.05               \\
		heart-hung.      & 0.75               & 0.75        & 0.71           & 0.85           & 0.14               \\
		mammogr.         & 0.77               & 0.77        & 0.57           & 0.78           & 0.21               \\
		monks-1          & 0.67               & 0.66        & 0.49           & 0.57           & 0.17               \\
		planning         & 0.53               & 0.52        & 0.53           & 0.7            & 0.17               \\
		spectf           & 0.64               & 0.64        & 0.56           & 0.58           & 0.08               \\
		horse-colic      & 0.75               & 0.59        & 0.74           & 0.85           & 0.26               \\
		ilpd-indian      & 0.59               & 0.59        & 0.53           & 0.72           & 0.19               \\
		parkinsons       & 0.74               & 0.74        & 0.65           & 0.88           & 0.23               \\
		pima             & 0.68               & 0.68        & 0.68           & 0.87           & 0.2                \\
		\bottomrule
	\end{tabular}
	\vspace{-2ex}
\end{table}

\begin{figure*}[t]
	\centering
	\includegraphics[width=0.98\textwidth]{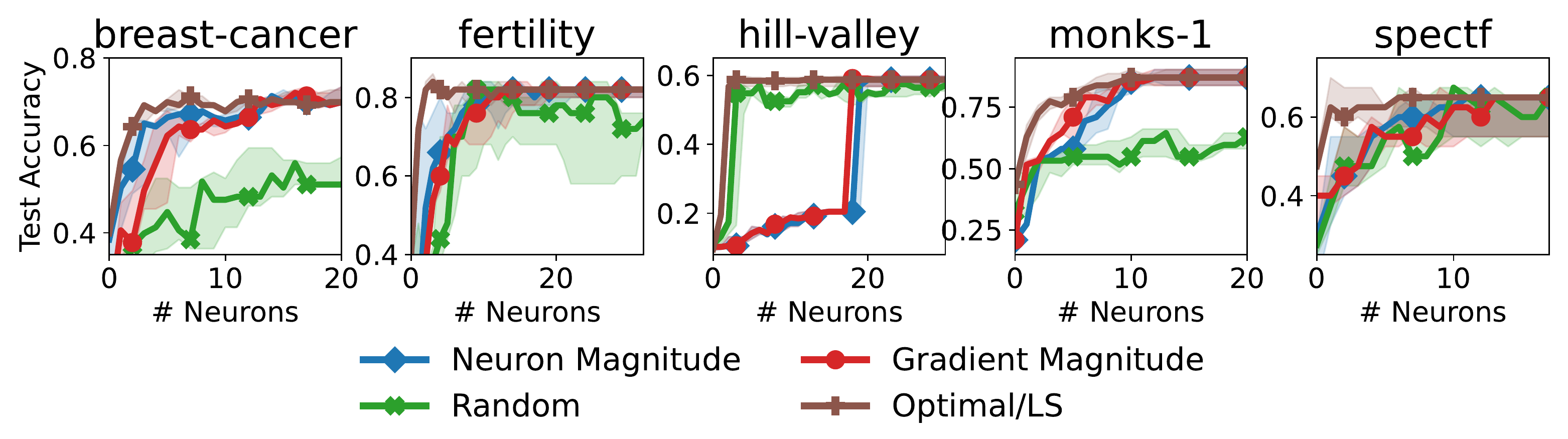}
	\vspace{-2.5ex}
	\caption{Pruning neurons from two-layer ReLU networks on binary
		classification tasks from the UCI repository.
		We compare our theory-inspired approach (\textbf{Optimal/LS}),
		against removing the neuron with smallest \( \ell_2 \) norm
		(\textbf{Neuron Magnitude}),
		removing the neuron with the smallest weighted gradient norm
		(\textbf{Gradient Magnitude}), and random pruning (\textbf{Random}).
		For Optimal/LS, we use \cref{alg:pruning-solutions-nn}, which begins
		with optimal pruning and then switches to a least-squares heuristic.
		We plot test accuracy against number of active neurons.
		Optimal/LS dominates the baseline methods on every dataset
		and even improves test accuracy on \texttt{breast-cancer}
		and \texttt{fertility}.
	}
	\label{fig:uci-pruning-acc}
	\vspace{-2ex}
\end{figure*}

\textbf{Pruning}:
If \( (v, u) \) is a minimal solution to the convex reformulation, then the
corresponding ReLU network is the p-unique model using only those
activation patterns (Propositions~\ref{prop:minimal-solutions} and~\ref{prop:p-unique}).
Thus, \cref{alg:pruning-solutions} can be used to prune
any solution to obtain a neuron-sparse neural network
achieving the optimal training objective.
\cref{alg:pruning-solutions-nn} specializes our pruning algorithm
to the ReLU problem and extends it to support approximate pruning.
Note the resulting procedure is completely independent of the convex
reformulation.
The complexity of this method is as follows.

\begin{restatable}{proposition}{minimalComplexity}\label{prop:minimal-complexity}
	Suppose \( r = \text{rank}(X) \).
	Then an optimal and minimal ReLU network with at most
	\( m^* \leq n \) non-zero neurons can be computed in
	\( O\rbr{d^3 r^3 (n/r)^{3r}} \)
	time.
\end{restatable}

As a consequence, the complexity of computing
an optimal and minimal ReLU network is fully polynomial when \( r \)
is bounded.
We also have a more sensitive statement for the minimal width: if
\( (W_1^*, w_2^*) \) are optimal weights for the ReLU model, then
\( m^* \) is at most the dimension of
\( \Span \cbr{(X W^*_{1i})_+}_i \)
and exactly this dimension under group dependence
(\cref{prop:minimal-sol-characterization}).
We experiment with pruning ReLU networks using this approach in
\cref{sec:experiments} and that show it
is more effective than naive pruning strategies.

\textbf{Continuity}:
First we give a negative result for singular networks, that is, models
where \( m < m^* \) and no convex reformulation exists.
In this setting, the solution map can be made to behave
arbitrarily poorly.
\vspace{-0.5ex}
\begin{restatable}{proposition}{oneNeuron}\label{prop:one-neuron}
	There exists \( (Z, y) \) for which \( \calO^* \)
	is not open nor is the model fit \( f_{W_1, w_2}(Z) \) continuous in
	\( \lambda \).
\end{restatable}
Combined with the next result, \cref{prop:one-neuron} indicates
that the threshold \( m^* \) may be crucial for continuity to extend
to the non-convex parameterization.
\begin{corollary}\label{cor:relu-continuity}
	Suppose \( m \geq m^* \).
	Then the optimal model fit for two-layer gated ReLU networks
	is continuous at all \( \lambda > 0 \).
	Similarly, if the (gated) ReLU solution is p-unique on an open
	interval \( \Lambda \), then the regularization path is also continuous
	on \( \Lambda \) up to permutations of the weights.
\end{corollary}
Together, \cref{cor:relu-continuity} and \cref{prop:relu-unique}
are concrete conditions for the model fit and regularization path
of a sub-sampled problem to be continuous.

\textbf{Min-Norm Solutions}:
In \cref{sec:min-norm}, we examined minimum \( \ell_2 \)-norm solutions to CGL.
However, all optimal ReLU models have the same
\( \ell_2 \)-norm when \( \lambda > 0 \).
Minimizing the Euclidean norm of solutions to the convex reformulation instead
minimizes the sum of norms to the fourth power.
\begin{restatable}{lemma}{minNormNN}\label{lemma:min-norm-nn}
	The minimum \( \ell_2 \)-norm solution to the convex reformulation
	of a (gated) ReLU model corresponds to the
	optimal neural network which minimizes,
	\[
		r(W_1, w_2) = \sum_{i = 1}^m \norm{W_{1i}}_2^4 + \norm{w_{2i}}_2^4,
		\vspace{-1ex}%
	\]
	and admits no neuron-merge symmetries.
\end{restatable}
\vspace{-1ex}%

As a result, we can compute the \( r \)-minimal optimal ReLU network
by solving Problem~\eqref{eq:min-norm-program}.
If \( \calS_\lambda \) is unknown, then using \( \act(w) \) for some solution
\( w \) as an approximation gives the \( r \)-minimal network
using a subset of those activations.

\textbf{Sensitivity}:
\cref{prop:local-sol-fn} extends similar results for the group lasso by
\citet{vaiter2012dof} to CGL using standard
CQs.
Since \( K \) is block-diagonal,
LICQ will be satisfied whenever the rows of \( Z \) are linearly independent.
SCS is more challenging; while the classical theorem of
\citet{goldman20164} establishes that SCS is satisfied for linear programs,
it is known that SCS can fail for general cone programs \citep{tunccel2012strong}.
As such, SCS must be checked on a per-problem basis in general.

In the context of gated ReLU problems, \( K = 0 \) and there is no requirement
for CSC/LICQ.
Minimal models \( w(\lambda, y) \) are weakly differentiable, which
\citet{vaiter2012dof} uses to compute the degrees of freedom of
\( w \) via Stein's Lemma \citep{stein1981estimation}.
It is straightforward to extend this calculation to the gated ReLU weights
using the chain rule, which can then be used to calculate
Stein's unbiased risk estimator.

\section{Experiments}\label{sec:experiments}


Through convex reformulations, we have characterized the optimal sets of ReLU
networks, computed minimal networks, and provided sensitivity results.
Our goal in this section is to illustrate the power of our framework for
analyzing ReLU networks and developing new algorithms.

\textbf{Tuning}:
We first consider a tuning task on 10 binary classification
datasets from the UCI repository~\citep{dua2019uci}.
For each dataset, we do a train/validation/test split, fit a two-layer
ReLU model on the training set, and then compute the minimum \( \ell_2 \)-norm
model.
We use this to explore the optimal set in three ways:
(i) we compute an extreme point that (approximately)
maximizes the model's \( \ell_2 \)-norm; (ii)
we minimize the validation MSE over
\( \solfn(\lambda) \); (iii) we minimize test MSE over \( \solfn(\lambda) \).
These procedures select for different optimal models,
have no effect on the training objective, and are only possible
because we know \( \solfn \).

The results are summarized in \cref{table:tuning-small}.
We see that optimal models can perform very differently at test time despite
having exactly the same training error and model norm.
Indeed, 9/10 datasets show at least a 10 percent gap between the best
and worst models and 4/10 have a gap exceeding 20 percent accuracy.
We conclude that the training objective is badly under-determined even for
shallow neural networks, implying that implicit regularization is critical
in practice.
See \cref{app:additional-experiments} for results on additional datasets.

\begin{figure}[t]
	\centering
	\includegraphics[width=0.48\textwidth]{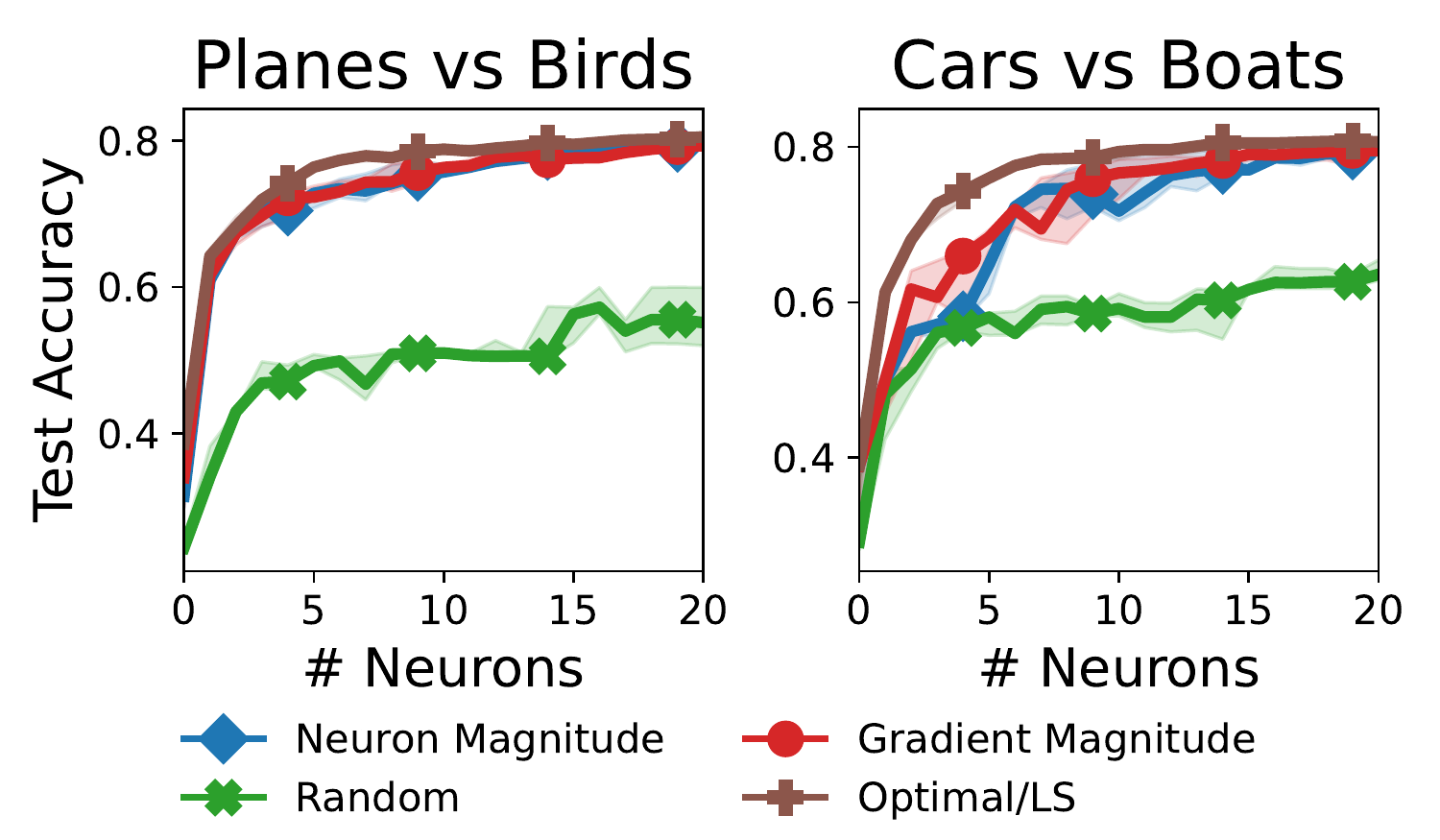}
	\vspace{-2.5ex}
	\caption{Pruning neurons from two-layer ReLU networks on two binary
		classification tasks drawn from the CIFAR-10 dataset.
		We compare our method (\textbf{Optimal/LS}) against baselines;
		see \cref{fig:uci-pruning-acc} for details.
		Our approach, which makes use of a weight correction after pruning,
		outperforms every baseline.
	}
	\label{fig:cifar-pruning-acc}
	\vspace{-2ex}
\end{figure}

\textbf{Pruning}:
We also consider several neuron pruning tasks.
We use two-layer ReLU networks and start pruning from the model given
by optimizing the convex reformulation.
We compare four strategies:
(i) pruning neurons optimally using
\cref{alg:pruning-solutions-nn} until \( \cbr{(XW_{1i})_+}_\act \) are linearly
independent and then approximately using least-squares fits;
and (ii) by removing the neuron with the smallest magnitude,
\( \norm{W_{1i} \cdot w_{2i}} \);
(iii) by removing the neuron with the smallest weighted gradient;
and (iv) by random pruning.

\cref{fig:uci-pruning-acc} shows test performance of the these methods for five
UCI datasets.
Our theory-based pruning method has better test performance
than the baselines on every dataset considered;
on \texttt{hill-valley}, the gap between our approach and
magnitude-based pruning is approximately \( 40\% \).
\cref{fig:cifar-pruning-acc} presents similar results for two binary tasks
taken from the CIFAR-10 dataset \citep{krizhevsky2009learning}.
We provide experiments on additional datasets, including MNIST
\citep{lecun1998gradient},
and experimental details in
\cref{app:additional-experiments}.

\section{Conclusion}\label{sec:conclusion}


We study the structure and properties of solution sets for shallow
neural networks with (gated) ReLU activations.
Unlike previous work, we avoid non-convexity of neural networks
by studying the constrained group lasso,
a generalized linear model which unifies the convex reformulations of both
ReLU and gated ReLU networks.
We derive analytical expressions for the optima and all stationary
points of the training objective for two-layer ReLU networks.
Building on this characterization, we develop
conditions for the optimal neural network to be p-unique,
an algorithm for optimal pruning of neural networks,
and sensitivity results.
We demonstrate the utility of our framework
in experiments on MNIST, CIFAR-10, and the UCI datasets.

There is still much work to do in this area.
For example, we conjecture that the min-norm CGL solution,
which corresponds to the network minimizing a fourth-power penalty,
always has a continuous regularization path.
More generally, it remains to extend our characterization of the solution set to
deeper networks and vector-output models.


\section*{Acknowledgements}

This work was supported by the National Science Foundation (NSF),
Grants ECCS-2037304, DMS-2134248; by the NSF CAREER
Award, Grant CCF-2236829; by the U.S. Army Research Office
Early Career Award, Grant W911NF-21-1-0242; by the Stanford
Precourt Institute; by the Stanford Research Computing Center;
and by the ACCESS—AI Chip Center for Emerging
Smart Systems through InnoHK, Hong Kong, SAR.
Aaron Mishkin was supported by NSF Grant DGE-1656518 and by
NSERC Grant PGSD3-547242-2020.
We thank Frederik Kunstner and Victor Sanches Portella for many
insightful discussions.

\clearpage
\newpage

\bibliography{refs.bib}

\bibliographystyle{icml2023}


\newpage
\appendix
\onecolumn

\section{Constrained Group Lasso: Proofs}\label{app:cgl}

\subsection{Describing the Optimal Set}


\begin{lemma}\label{lemma:unique-fit-cgl}
	The model fit is the same for all optimal solutions to the
	CGL problem.
	That is,
	\[
		X w = X w' \quad \text{ for all } w, w' \in \solfn(\lambda).
	\]
	As a consequence, the sum of group norms is also constant
	when \( \lambda > 0 \).
\end{lemma}
\begin{proof}
	This follows in a similar fashion to the classic result for the group lasso.
	Let \( w, w' \in \solfn(\lambda) \), \( \bar w = \half w + \half w' \),
	and suppose \( p^* \) is the optimal value of the constrained program.
	By convexity, we have
	\begin{align*}
		\half \norm{X \bar w - y}_2^2 + \lambda \sum_{\bi \in \calB} \norm{\bwi}_2
		 & \leq
		\half p^* + \half p^* = p^*,
	\end{align*}
	where the inequality is strict if \( X w \neq X w' \) by strong convexity
	of \( f(u) = \norm{u - y}_2^2 \).
	Since \( w, w' \) are both feasible, \( \bar w \) is also feasible and
	clearly \( \bar w \) cannot obtain an objective value less than \( p^* \).
	Thus, \( X w = X w' \)  must hold.

	To see the second part of the result, observe that
	\[
		\lambda \sum_{\bi \in \calB} \norm{\bwi}_2 = p^* - \half \norm{\hat y(\lambda)  - y}_2^2,
	\]
	is also constant over \( \solfn(\lambda) \).
\end{proof}

\solFnCGL*
\begin{proof}
	Fix \( \lambda > 0 \) and let
	\( w \in \solfn(\lambda) \).
	If \( \wi \neq 0 \), then the KKT conditions require
	\[
		\lambda \frac{\wi}{\norm{\wi}_2} = \vi \implies \wi = \alpha_\bi \vi,
	\]
	for \( \alpha > 0 \).
	If \( \wi = 0 \), then \( \wi = \alpha_\bi \vi \) holds trivially for
	\( \alpha_\bi = 0 \).
	Since \( w \) is optimal, it must satisfy
	\[
		X \w = \hat y,
	\]
	by \cref{lemma:unique-fit-cgl}.
	Finally, \( \act(w) \subseteq \calS_\lambda \) so that
	\( w \) satisfies the characterization.

	For the reverse direction, we start by defining
	\[
		\begin{aligned}
			\calX =
			\bigg\{ w \in \R^d : \,
			 & \forall \, \bi \in \calS_\lambda, \,
			\wi = \alpha_\bi \vi, \alpha_\bi \geq 0, \,          \\
			 & \forall \, b_j \in \calB \setminus \calS_\lambda,
			\w_{b_j} = 0, \, X w = \hat y
			\bigg\}
		\end{aligned}
	\]
	Take any \( w' \in \calX \).
	If \( \wi' \neq 0 \), then \( \bi \in \equi \) and \( \norm{\vi}_2 = \lambda \)
	so that we may write
	\begin{align*}
		\wi' = \alpha_\bi \vi
		 & \implies \frac{\wi'}{\norm{\wi'}_2} = \frac{\alpha_\bi \vi}{\norm{\alpha_\bi \vi}_2} \\
		 & \implies \frac{\wi'}{\norm{\wi'}_2} = \frac{\vi}{\lambda}                            \\
		 & \implies \lambda \frac{\wi'}{\norm{\wi'}_2} = \vi.
	\end{align*}
	That is,
	\[
		\Xbi^\top (Xw - y) + \lambda \frac{\wi'}{\norm{\wi'}_2} + \Ki \rmin_\bi = 0,
	\]
	which is exactly stationarity of the Lagrangian.

	If \( \wi' = 0 \), then \( X w' = \hat y \) implies
	\[
		\norm{\Xbi^\top (y - Xw') + K \rmin_\bi}_2
		= \norm{\Xbi^\top (y - Xw) + K \rmin_\bi}_2
		\leq \lambda,
	\]
	for some \( w \in \solfn(\lambda) \),
	which also implies the Lagrangian is stationary.

	Now we show optimality of
	\( w' \) by checking feasibility and complementary slackness.
	If \( \wi' \neq 0  \)
	then \( \wi' = \frac{\alpha_\bi'}{\alpha_\bi} \wi \)
	for some optimal solution \( w \) with \( \alpha_\bi > 0 \).
	This follows since \( \wi' \neq 0 \) implies \( \bi \in \calS_\lambda \).
	Thus,
	\[
		\Ki^\top \wi' = \frac{\alpha_\bi'}{\alpha_\bi} \wi \leq 0,
	\]
	by feasibility of \( \wi \).
	Similarly, we find that complementary slackness is satisfied as follows:
	\[
		[\ri]_j \cdot [\Ki]_j^\top \wi'
		= \frac{\alpha_\bi'}{\alpha_\bi} [\ri]_j \cdot [\Ki]_j^\top \wi
		= 0.
	\]
	If \( \wi = 0 \), then both feasibility and complementary slackness
	are trivial.
	Since \( (w', \rmin) \) are feasible and \( \rmin \) is dual optimal,
	we conclude the KKT conditions are satisfied and thus
	\( w' \in \solfn(\lambda) \).
	This completes the proof.
\end{proof}

\begin{proposition}\label{prop:cgl-alternate}
	Fix \( \lambda > 0 \).
	The optimal set for CGL problem is also
	given by
	\begin{equation}\label{eq:sol-fn-cgl-alternate}
		\begin{aligned}
			\solfn(\lambda) =
			\bigg\{ w \in \R^d : \,
			 & \forall \, \bi \in \equi, \,
			\wi = \alpha_\bi \vi, \alpha_\bi \geq 0, \,                   \\
			 & \forall \, b_j \in \calB \setminus \equi, \w_{b_j} = 0, \,
			X w = \hat y,                                                 \\
			 & \, K^\top w \leq 0, \,
			\abr{\rmin, K^\top w} = 0
			\bigg\}
		\end{aligned}
	\end{equation}
\end{proposition}
\begin{proof}
	Fix \( \lambda > 0 \) and let
	\( w \in \solfn(\lambda) \).
	If \( \wi \neq 0 \), then the KKT conditions require
	\[
		\lambda \frac{\wi}{\norm{\wi}_2} = \vi \implies \wi = \alpha_\bi \vi,
	\]
	for \( \alpha > 0 \).
	If \( \wi = 0 \), then \( \wi = \alpha_\bi \vi \) holds trivially for
	\( \alpha_\bi = 0 \).
	Since \( w \) is optimal, it must satisfy
	\[
		X \w = \hat y,
	\]
	by \cref{lemma:unique-fit-cgl}.
	Finally, \( \abr{\rmin, K w} = 0 \) and is feasible by KKT conditions
	so that \( w \) satisfies the characterization.

	For the reverse direction, we start by defining
	\[
		\begin{aligned}
			\calX =
			\bigg\{ w \in \R^d : \,
			 & \forall \, \bi \in \equi, \,
			\wi = \alpha_\bi \vi, \alpha_\bi \geq 0, \,                   \\
			 & \forall \, b_j \in \calB \setminus \equi, \w_{b_j} = 0, \,
			X w = \hat y,                                                 \\
			 & \, K^\top w \leq 0, \,
			\abr{\rmin, K^\top w} = 0
			\bigg\}
		\end{aligned}
	\]
	Take any \( w' \in \calX \).
	If \( \wi' \neq 0 \), then \( \bi \in \equi \) and \( \norm{\vi}_2 = \lambda \)
	so that we may write
	\begin{align*}
		\wi' = \alpha_\bi \vi
		 & \implies \frac{\wi'}{\norm{\wi'}_2} = \frac{\alpha_\bi \vi}{\norm{\alpha_\bi \vi}_2} \\
		 & \implies \frac{\wi'}{\norm{\wi'}_2} = \frac{\vi}{\lambda}                            \\
		 & \implies \lambda \frac{\wi'}{\norm{\wi'}_2} = \vi.
	\end{align*}
	That is,
	\[
		\Xbi^\top (Xw - y) + \lambda \frac{\wi'}{\norm{\wi'}_2} + \Ki \rmin_\bi = 0,
	\]
	which is exactly stationarity of the Lagrangian.

	If \( \wi' = 0 \), then
	\[
		X w' = \hat y \implies \norm{\Xbi^\top (y - Xw') + K \rmin_\bi}_2 \leq \lambda,
	\]
	which also implies the Lagrangian is stationary.

	Since \( \abr{\rmin, K^\top w'} = 0 \) and
	\( K^\top w' \leq 0 \), i.e. \( w' \) is feasible,
	it is clear that complementary slackness must also hold.
	We conclude that \( (w', \rmin) \) are primal-dual optimal by the KKT
	conditions and the proof is complete.
\end{proof}

\uniqueCGL*
\begin{proof}
	Suppose by way of contradiction that \( w, w' \in \solfn(\lambda) \)
	such that \( w \neq w' \).
	By \cref{prop:sol-fn-cgl}, we have
	\begin{align*}
		0 = X (w - w')
		 & = \sum_{\bi \in \calS_\lambda} \Xbi (\wi - \wi')                    \\
		 & = \sum_{\bi \in \calS_\lambda} (\alpha_\bi - \alpha_\bi') \Xbi \vi,
	\end{align*}
	which implies that the vectors \( \cbr{\Xbi \vi}_{\calS_\lambda} \) are
	linearly dependent.

	Necessity follows from \cref{alg:pruning-solutions},
	which shows that, given a solution \( w \in \solfn(\lambda) \),
	linear dependence of \( \cbr{\Xbi w(\lambda) : \bi \in \calA(w(\lambda)) } \)
	implies the exist of at least one additional solution.
\end{proof}

\GGP*
\begin{proof}
	Suppose the group Lasso solution is not unique.
	Then, \cref{cor:uniqueness-cond-cgl} implies
	\[
		\calN_\lambda = \Null(\Xe) \bigcap \cbr{z : \zi = \alpha_\bi \ci, \bi \in \equi },
	\]
	is non-empty.
	That is, there exist \( \alpha_{\bi} \geq 0 \) such that
	\begin{align*}
		X_{b_j} c_{b_j}
		 & = \sum_{\bi \in \equi \setminus j} \alpha_{\bi} \Xbi \ci.
		\intertext{Taking inner-products on both sides with the residual
			\( r \),}
		r^\top X_{b_j} c_{b_j}
		 & = \sum_{\bi \in \equi \setminus j} \alpha_{\bi} r^\top \Xbi \ci \\
		\implies
		c_{b_j}^\top c_{b_j}
		 & = \sum_{\bi \in \equi \setminus j} \alpha_{\bi} \ci^\top \ci    \\
		\implies \lambda^2
		 & = \sum_{\bi \in \equi \setminus j} \alpha_{\bi} \lambda^2       \\
		\implies 1
		 & = \sum_{\bi \in \equi \setminus j} \alpha_{\bi}.
	\end{align*}
	Thus, we deduce that
	\begin{equation}\label{eq:dependence-relation}
		X_{b_j} c_{b_j}
		= \sum_{\bi \in \equi \setminus j} \beta_{\bi} \Xbi \ci,
	\end{equation}
	where \( \sum_{\bi \in \equi \setminus j} \beta_{\bi} = 1 \),
	meaning \( X_{b_j} c_{b_j} \in \text{affine}(\Xbi \zi : \bi \in \equi \setminus b_j) \)
	Now, suppose that \( |\equi| > n + 1 \).
	Then, \( \cbr{\Xbi \ci : \bi \in \equi \setminus j} \) are linearly dependent
	and, by eliminating dependent vectors \( \Xbi \ci \),
	we can repeat the above proof with a subset \( \calE' \) of at most \( n + 1 \)
	blocks.
	Noting \( \norm{\ci}_2 = \lambda \) for each \( \bi \in \equi \) and
	rescaling both sides of \cref{eq:dependence-relation} by \( \lambda \)
	implies the existence of unit vectors \( z_{\bi} \) which contradict
	GGP.
	This completes the proof.
\end{proof}

\begin{proposition}\label{prop:general-position-relation}
	Group general position does not imply
	the columns of \( X \) are in general position.
	Similarly, general position of the columns of \( X \) does not imply
	group general position.
\end{proposition}
\begin{proof}
	Let \( x_1, \ldots, x_{d} \in \R^{d-1} \).
	Consider the simple case where we have two groups: \( b_1 = \cbr{1} \)
	and \( b_2 = \cbr{2, \ldots, d} \).
	Group general position is violated if there exists a unit vector \( z_{b_2} \)
	such that
	\begin{align*}
		x_{1}
		 & = X_{b_2} z_{b_2}.   \\
		\iff x_1
		 & \in X_{b_2} B_{d-1},
	\end{align*}
	where \( B_{d-1} = \cbr{z \in \R^{d-1} : \norm{z}_2 \leq 1} \).
	In contrast, general position is violated if
	\begin{align*}
		x_1
		 & \in \text{affine}(x_2, \ldots, x_{d}) \\
		\iff x_1
		 & \in X_{b_2} \cbr{z : \abr{z, 1} = 1}.
	\end{align*}
	Taking \( X_{b_2} = I \), it is trivial to see that group general position can
	hold when general position is violated and vice-versa.
\end{proof}

\subsection{Computing Dual Optimal Parameters}


\begin{lemma}\label{lemma:lagrange-dual}
	One Lagrange dual of CGL is the following:
	\begin{equation}\label{eq:dual-program}
		\begin{aligned}
			\max_{\eta, \rho} \;
			 & - \half (\eta + K \rho - X^\top y) (X^\top X)^+ (\eta + K \rho - X^\top y) + \half \norm{y}_2^2                         \\
			 & \quad \quad \text{s.t.} \quad \eta + K \rho \in \Row(X), \; \norm{\eta_\bi}_2 \leq \lambda \; \forall \; \bi \in \calB,
		\end{aligned}
	\end{equation}
	where \( \eta_\bi = \ci - \Ki \ri \) shows that the vectors \( \vi \)
	are, in fact, dual variables.
	Moreover, if \( K = 0 \), then \( \rmin = 0 \) and \( \eta \)
	has the unique solution \( \eta_\bi = \Xbi^\top (y -  Xw) = \ci \).
	That is, the dual parameters are the (unique) block correlation vectors.
\end{lemma}
\begin{proof}
	We re-write the group Lasso problem as follows:
	\begin{align*}
		\min_{w} \half \norm{X z - y}_2^2 + \lambda \sum_{\bi \in \calB} \norm{\wi}_2
		\quad \text{s.t.} \, \, z = w, \, \Ki^\top z_\bi \leq 0.
	\end{align*}
	The Lagrangian for this problem is
	\begin{align*}
		\calL(w, z, \eta, \rho)
		 & = \half \norm{X z - y}_2^2 + \abr{\eta, z - w} + \abr{\rho, K^\top z} + \lambda \sum_{\bi \in \calB} \norm{\wi}_2 \\
		 & = \half \norm{X z - y}_2^2 + \abr{\eta + K \rho, z} - \abr{\eta, w} + \lambda \sum_{\bi \in \calB} \norm{\wi}_2.
	\end{align*}
	Minimizing over \( z \), we find that stationarity implies
	\[
		X^\top (y - X z) = \eta + K \rho,
	\]
	so that \( \eta + K \rho \in \Row(X) \).
	Solving this system, we find
	\[
		X^\top X z = X^\top y - \eta - K \rho \implies z = (X^\top X)^+ \sbr{X^\top y - \eta - K \rho} + c,
	\]
	where \( c \in \Null(X) \).
	Let us minimize over \( w \) similarly.
	The Lagrangian decouples block-wise in \( w \), so that
	we must solve
	\[
		\min_{\wi} \lambda \norm{\wi}_2 - \abr{\eta_\bi, \wi},
	\]
	for each \( \bi \in \calB \).
	The minimum value is achieved by the (negative) Fenchel conjugate of
	\( \lambda \norm{\wi}_2 \)
	evaluated at \( \eta_\bi \);
	that is,
	\[
		\min_{\wi} \lambda \norm{\wi}_2 - \abr{\eta_\bi, \wi} = - \mathbbm{1}(\norm{\eta_\bi}_2 \leq \lambda).
	\]
	Combining this with the expression for \( z \), we obtain
	\begin{align*}
		\calL(c, \eta, \rho)
		 & = - \half (\eta + K \rho - X^\top y) (X^\top X)^+ (\eta + K \rho - X^\top y) + \abr{\eta + K \rho, c}
		- \sum_{\bi \in \calB} \mathbbm{1}(\norm{\eta_\bi}_2 \leq \lambda)                                       \\
		 & = - \half (\eta + K \rho - X^\top y) (X^\top X)^+ (\eta + K \rho - X^\top y)
		- \sum_{\bi \in \calB} \mathbbm{1}(\norm{\eta_\bi}_2 \leq \lambda),
	\end{align*}
	where the second equality follows since \( \eta + K \rho \in \Row(X) \) and \( c \in \Null(X) \)
	are orthogonal.
	(Alternatively, one can observe that the dual problem is unbounded below
	whenever \( \abr{c, \eta + K \rho} \neq 0 \).)
	Thus, the dual problem is equal to
	\begin{align*}
		\max_{\eta, \rho} - \half (\eta + K \rho - X^\top y) (X^\top X)^+ (\eta _ K \rho - X^\top y)
		- \sum_{\bi \in \calB} \mathbbm{1}(\norm{\eta_\bi}_2 \leq \lambda)
		- \mathbbm{1}(\eta + K \rho \in \Row(X)),
	\end{align*}
	which completes the derivation.

	Recalling \( z = w \) for any primal-dual optimal pair and
	\[
		X^\top (y - X z) = \eta + K \rho,
	\]
	shows that \( \eta_\bi = \ci - K \rho \) as claimed.
	Moreover, if \( K = 0 \), then we may assume without loss of generality
	that he corresponding dual vectors \( \ri \) are zero.
	In this case, \( \eta_\bi = \ci \) and,
	since \( \ci \) is unique, the dual solution must also be unique.
\end{proof}

\begin{proposition}\label{prop:non-negative-regression}
	Let \( \lambda > 0 \) and \( w \in \solfn(\lambda) \).
	If \( \wi = 0 \), then any solution to \cref{eq:non-negative-regression}
	is dual optimal for block \( \bi \).
\end{proposition}
\begin{proof}
	Let \( \ri \) be a solution to \cref{eq:non-negative-regression}.
	Since \( w \) is optimal and strong duality holds, there exists
	some min-norm dual optimal vector \( \rmin \).
	Moreover \( \rmin \) satisfies \( \rmin_\bi \geq 0 \) and
	\[
		\norm{\Ki \ri - \ci}^2_2 \leq \norm{\Ki \rmin_\bi - \ci}^2_2 \leq \lambda^2,
	\]
	so \( w \) is both feasible satisfies stationarity of the Lagrangian,
	Finally, because \( \wi = 0 \), complementary slackness,
	\[
		[\ri]_j \cdot [\Ki]_j^\top \wi = 0,
	\]
	is verified for every \( j \in [a_\bi] \).
	Since the KKT conditions are sufficient for primal-dual optimality, we
	conclude that \( \ri \) is dual optimal.
	This completes the proof.
\end{proof}

\subsection{Minimal Solutions and Optimal Pruning}


\minimalSols*
\begin{proof}
	Let \( w \in \solfn(\lambda) \) and assume that
	the vectors \( \cbr{\Xbi \wi}_{\calA(w)} \) are linearly independent.
	By way of contradiction, assume there exists \( w' \in \solfn(\lambda) \)
	with strictly smaller support.
	By \cref{prop:sol-fn-cgl}, we have
	\[
		\wi' = \beta_\bi \wi,
	\]
	for some \( \beta_\bi \geq 0 \).
	This holds for each \( \bi \in \act(\w) \)
	(with \( \beta_\bi = 0 \) when \( \bi \in \act(w) \setminus \act(w') \)) so that
	\[
		X w = X w' \implies \sum_{\bi \in \act(w)} (1 - \beta_\bi) \Xbi \wi = 0,
	\]
	which is a contradiction.

	For the reverse direction, assume that \( w \) is minimal, but that
	\( \cbr{\Xbi \wi}_{\calA(w)} \) are not linearly independent.
	Then the correctness of \cref{alg:pruning-solutions}
	(see \cref{prop:pruning-correctness}) implies
	\( w \) is not minimal.
\end{proof}

\begin{proposition}\label{prop:pruning-correctness}
	\cref{alg:pruning-solutions} returns a minimal solution to the
	constrained group lasso problem in at most \( O(n^3l + nd) \)
	time, where \( l \) is the number of non-zero groups in the
	initial solution.
\end{proposition}
\begin{proof}
	\textbf{Correctness}:
	Let \( w \in \solfn(\lambda) \)
	and \( \calA \) be the associated active set.
	If \( w^0 = w \) is minimal, then \cref{prop:minimal-solutions}
	implies \( \cbr{\Xbi \wi}_\calA \) are linearly independent
	the algorithm returns a minimal solution.

	Let \( k \geq 0 \) and suppose \( w^k \) is not minimal.
	Then there exist weights \( \beta_\bi \) such that
	\[
		\sum_{\bi \in \calA} \beta_\bi \Xbi w^k = 0.
	\]
	Let \( \wi^t = (1 - t \beta_\bi) \wi \) and let \( t^k \)
	be as defined in the algorithm.
	By construction, \( t^k \) is the smallest magnitude \( t \) such that
	\( (1 - t \beta_\bi) = 0 \) for some \( \bi \in \calA \).
	We assume without loss of generality that \( t^k > 0 \).

	Fix \( 0 < t < t^k \). Let's show that \( w^t \) is a solution to the
	constrained problem.
	Firstly, we have
	\[
		X w^t = X w^k - t \sum_{\bi \in \calA} \beta_\bi \Xbi \wi^k = X w^k,
	\]
	showing that the model fit preserved.
	Moreover,
	\[
		\wi^t = (1 - t \beta_\bi) \wi^k = (1 - t \beta_\bi) \alpha_\bi \vi,
	\]
	where \( (1 - t \beta_\bi) \alpha_\bi > 0 \) by the choice of \( t \).
	We conclude that \( w^t \) is optimal by \cref{prop:sol-fn-cgl}.

	By construction,
	\[
		\lim_{t \uparrow t^k} w^t = w^{k+1}.
	\]
	Since \( w^t \) is an optimal solution,
	it has the (unique) optimal squared error and sum of group norms.
	Taking limits as \( t \uparrow t^k \),
	we see that
	\[
		X w^{k+1} = X w^k,
		\quad \sum_{\bi \in \calB} \norm{\wi^{k+1}}_2
		=
		\sum_{\bi \in \calB} \norm{\wi}_2,
		\quad
		K^\top w^{k+1} \leq 0,
	\]
	which implies that \( w^{k+1} \) obtains the optimal objective
	value and is feasible.
	Thus \( w^{k+1} \) is also a solution.
	Finally, \( \calA(w^{k+1}) \) is strictly smaller than
	\( \calA(w^k) \), as required.

	Arguing by induction now implies that \cref{alg:pruning-solutions}
	returns an minimal solution in a finite number of steps.

	\textbf{Complexity}:
	First, observe that we can pre-compute the block-wise model fits
	\( q_\bi = \Xbi \wi \) before running the
	algorithm. The complexity is at most \( O(nd) \).
	At each iteration of the algorithm, we must do two things: (i)
	compute a non-trivial solution to a homogeneous equation and (ii)
	update the weights of the model fits.
	For (i), it is clear that any set of \( n+1 \) of the \( q_\bi \) vectors
	will be linearly dependent, so that we compute a non-trivial
	solution to the homogeneous equation using the SVD in an at most
	\( O(n^3) \) operations.
	For (ii), updating at most \( n \) of the \( \beta_{\bi} \)'s
	requires \( O(n) \) time.
	Since the algorithm runs at most \( l \) iterations, we obtain a final
	complexity of \( O(n^3l + nd) \), as claimed.
\end{proof}

\begin{definition}\label{def:group-dependence}
	We say that \( X \) satisfies \emph{group dependence} if for every choice of
	vectors \( z_{\bi} \in \R^{|\bi|} \) and partition
	\( \calD \cup \calI = \calB \), \( \calD \cap \calI = \emptyset \) such
	that
	\[
		g_{\calD} := \sum_{b_j \in \calD} X_{b_j} z_{b_j}
		\in \text{Span}\rbr{\cbr{\Xbi z_{\bi} : \bi \in \calI}},
	\]
	it holds that
	\[
		\text{Span}\rbr{\cbr{X_{b_j} z_{b_j} : b_j \in \calD}}
		\subseteq \text{Span}\rbr{\cbr{\Xbi z_{\bi} : \bi \in \calI}}.
	\]
	In other words, linear dependence of
	\( g_{\calD} \)  and \( \cbr{\Xbi z_{\bi} : \bi \in \calI} \) implies
	linear dependence of \( X_{b_j} z_{b_j} \) and
	\( \cbr{\Xbi z_{\bi} : \bi \in \calI} \) for every \( b_j \in \calD \).
\end{definition}

\begin{lemma}\label{lemma:pruning-span}
	If \( X \) satisfies group dependence, then each step of the
	\cref{alg:pruning-solutions} preserves the span of \( \cbr{\Xbi \wi^k } \).
	That is, only linearly dependent vectors are removed.
\end{lemma}
\begin{proof}
	Let \( \calV = \Span\cbr{\Xbi \wi } \) be the span of the initial solution.
	Since \( w^0 = w \) by definition, the base case holds trivially.

	Now, suppose \( \Span(\Xbi \wi^{k}) = \calV \) and let
	\( \calD^k = \act(w^{k}) \setminus \act(w^{k+1}) \).
	For every \( b_j \in \calD^k \), it holds that \( \beta_{b_j} = \beta_{\bi^k} \),
	meaning
	\[
		\sum_{j \in \calD^k} X_{b_j} w_{b_j}^k \in \text{Span}\rbr{\cbr{\Xbi \wi^k : \bi \in \act(w^{k+1})}}.
	\]
	Group dependence now implies for every \( b_j \in \calD^k \) there exists \( \zeta^j \)
	such that
	\[
		X_{b_j} w_{b_j}^k = \sum_{\bi \in \act(w^k)} \zeta^j_{\bi} \Xbi \wi^k.
	\]
	Let \( v \in \calV \) and observe that
	\begin{align*}
		v
		 & = \sum_{\bi \in \act(w^k)} \alpha_\bi \Xbi \wi^k                  \\
		 & = \sum_{\bi \in \act(w^{k+1})} \alpha_\bi \Xbi \wi^k
		+ \sum_{b_j \in \calD^k} \alpha_{b_j} X_{b_j} w_{b_j}^{k}            \\
		 & = \sum_{\bi \in \act(w^{k+1})} \alpha_\bi \Xbi \wi^k
		+ \sum_{b_j \in \calD^k}
		\sum_{\bi \in \act(w^{k+1})} \alpha_{b_j} \zeta_{b_i}^j \Xbi \wi^{k} \\
		 & = \sum_{\bi \in \act(w^{k+1})}
		\rbr{\alpha_\bi + \sbr{\sum_{b_j \in \calD^k} \alpha_{b_j} \zeta_{b_i}^j}}
		\Xbi \wi^{k}                                                         \\
		 & = \sum_{\bi \in \act(w^{k+1})}
		\rbr{\frac{\alpha_\bi + \sbr{\sum_{b_j \in \calD^k}
					\alpha_{b_j} \zeta_{b_i}^j}}{1 - t^k \beta_\bi}}
		\Xbi \wi^{k+1},
	\end{align*}
	Thus, \( \Span(\Xbi \wi^{k+1}) = \calV  \).
	Arguing by induction completes the proof.
\end{proof}

\begin{lemma}\label{lemma:pruning-paths-equiv}
	Let \( \calX = \cbr{x_1, \ldots, x_k} \) be a set of linearly dependent vectors.
	Every linearly independent subset of \( \calX \) obtained by iteratively removing
	linearly dependent vectors has the same cardinality.
\end{lemma}
\begin{proof}
	Let \( \calY, \calY' \subset \calX \) be linearly independent subsets
	obtained by pruning linearly dependent vectors from \( \calX \) and assume
	\( |\calY'| < |\calY| \).
	Since \( \calY \) and \( \calY' \) are obtained by pruning only linearly
	dependent vectors, it must be that
	\[ \Span\rbr{\calY'} = \Span\rbr{\calY} = \Span\rbr{\calX}. \]
	Let \( c = \text{dim}\rbr{\Span\rbr{\calX}} \).
	Only a set of \( c \) linearly independent vectors can span
	\( \Span\rbr{\calX} \);
	thus, \( |\calY| = c \) must hold and \( \calY' \) cannot span \( \Span\rbr{\calX} \).
	This is a contradiction.
	We conclude
	\( |\calY'| = |\calY| \)
	as claimed.
\end{proof}

\begin{lemma}\label{lemma:maximal-sol}
	There exists a solution to CGL with support exactly \( \calS_\lambda \).
\end{lemma}
\begin{proof}
	By definition, there exists \( w \in \solfn(\lambda) \) such that
	\( \wi \neq 0 \) for every \( \bi \in \calS_\lambda \).
	Taking convex combination of these solutions yields \( w' \) with support
	exactly \( \calS_\lambda \).
	Since \( \solfn(\lambda) \) is convex, \( w' \) is also a solution.
	This completes the proof.
\end{proof}

\begin{lemma}\label{lemma:one-step-pruning}
	Let \( w \) be a minimal solution and \( \bar w \) be a solution with support
	\( \calS_\lambda \), which exists by \cref{lemma:maximal-sol}.
	Then, \( \bar w \) can be pruned in one step to obtain \( w \).
\end{lemma}
\begin{proof}
	Suppose \( \bar w \) is minimal.
	Then \( \cbr{\Xbi \bwi}_{\calS_\lambda} \) are linearly independent,
	which implies \( \cbr{\Xbi \vi}_{\calS_\lambda} \) are also linearly independent.
	We conclude \( \bar w \) is the unique solution to CGL by \cref{lemma:unique-cgl}
	and the claim holds trivially.

	Suppose \( \bar w \) is not minimal.
	Since \( \act(w) \subset \calS_\lambda \),
	\( \Span\cbr{\Xbi \wi} \subset \Span\cbr{\Xbi \bwi} \)
	so that every vector \( \Xbi \wi \) can be written as a linear combination
	of vectors in \( \cbr{\Xbi \bwi} \).
	Thus, we find that
	\begin{align*}
		\sum_{\bi \in \act(\bar w)} \Xbi \bwi
		 & = \hat y
		= \sum_{\bi \in \act(w)} \Xbi \wi \\
		\implies
		\sum_{\bi \in \calS_\lambda \setminus \act(w)} \Xbi \bwi +
		\sum_{\bi \in \act(w)} \beta_{\bi} \Xbi \bwi
		 & = 0,
	\end{align*}
	where \( \beta_{\bi} = 1 - \frac{\alpha_\bi}{\bar \alpha_\bi} \) for
	\( \wi = \alpha_\bi \vi \),
	\( \bwi = \bar \alpha_\bi \vi \).
	Thus, \( \cbr{\Xbi \bwi} \) are linearly dependent and it is possible
	to prune the solution.

	Now we show that we can, in fact, prune all vectors in \( \calS_\lambda \setminus \act(w) \)
	in one pruning step.
	First, observe that \( \frac{\alpha_\bi}{\bar \alpha_\bi} > 0 \) so that \( \beta_\bi < 1 \)
	for every \( \bi \in \act(w) \).
	Following the proof of \cref{prop:pruning-correctness},
	define
	\[
		\wi^t = (1 - t \beta_\bi) \bwi = (1 - t \beta_\bi) \bar \alpha_\bi \vi,
	\]
	for \( 0 < t < 1 \).
	Since \( \beta_\bi < 1 \) for every \( \bi \in \calS_\lambda \),
	it is straightforward to deduce that \( \wi^t \) is optimal by
	\cref{prop:sol-fn-cgl}.
	Arguing as in \cref{prop:pruning-correctness}, we can show
	\( w^1 \) is also optimal.
	It remains only to notice that \( \wi^1 = 0 \) for \( \bi \in \calS_\lambda \setminus \act(w) \)
	and \( \wi^1 = \wi \) for \( \bi \in \act(w) \).
	Thus, the pruning algorithm can move from \( \bar w \) to \( w \) in one step.
\end{proof}

\minimalSolCharacterization*
\begin{proof}
	Suppose \( w \) and \( w' \) are two minimal solutions.
	Both can be obtained by pruning the maximal solution with support \( \calS_\lambda \) (\cref{lemma:one-step-pruning}).
	Thus, \( w \) and \( w' \) both span \( \Span(\cbr{\Xbi \wi}) \) by \cref{lemma:pruning-span}.
	\cref{lemma:pruning-paths-equiv} now implies \( w \) and \( w' \)
	have the same number of active blocks.
	This number must be \( c \), otherwise there would be a linearly dependent vector
	in \( \cbr{\Xbi \wi}_{\act\rbr{w}} \) and \( w \) would not be minimal.
	This completes the proof.
\end{proof}

\subsection{Continuity of the Solution Path}


\begin{restatable}{lemma}{boundedSolutionsCGL}\label{lemma:bounded-solutions-cgl}
	Every solution to the constrained group lasso problem
	is bounded by an absolute constant independent of \( \lambda \).
	Specifically, every \( w \in \solfn(\lambda) \) satisfies
	\[
		\sum_{\bi \in \calB} \norm{\wi}_2 \leq \sum_{\bi \in \calB} \norm{\bar w_\bi}_2.
	\]
	where \( \bar w \in \solfn(0) \) is the least-squares solution with minimum \( \ell_2 \)-norm.
\end{restatable}
\begin{proof}
	Let \( h(w) = \sum_{\bi \in \calB} \norm{\wi}_2 \),
	and define \( W_{g} \) to be the set of least squares solutions with
	minimum group norm,
	\[
		W_{g} = \argmin \cbr{h(w) : w \in \solfn(0) }.
	\]
	Let \( w_g \in W_g \),
	and suppose that \( h(w_{g}) < h(w(\lambda)) \)
	for some \( \lambda > 0 \), \( w \in \solfn(\lambda) \).
	Since
	\[
		\half \norm{X w_g - y}_2^2
		\leq \half \norm{X w(\lambda) - y}_2^2
	\]
	we deduce
	\[
		\half \norm{X w_g - y}_2^2 + \lambda h(w_g) <
		\half \norm{X \w(\lambda) - y}_2^2
		+ \lambda h(\w(\lambda)),
	\]
	which is a contradiction.
	So \( h(\w(\lambda)) \leq h(w_g) \) for all \( \lambda > 0 \).
	Observing \( h(w_g) \leq h(\bar w) \) since \( \bar w \) may not be in
	\( W_{g} \) gives the result.
	Since \( h(\bar w) \) is independent of \( \lambda \),
	we conclude that \( \solfn(\lambda) \) is bounded independent
	of \( \lambda \).
\end{proof}

\valueContinuityCGL*
\begin{proof}
	Define the joint objective function
	\[
		f(w, \lambda) = \half \norm{X w - y}_2^2 + \lambda \sum_{b_i \in \calB} \norm{\wi}_2.
	\]
	Clearly \( f(w, \lambda) \) is jointly continuous in \( w \) and \( \lambda \).
	By \cref{lemma:bounded-solutions-cgl}, minimization of
	\( f(w, \lambda) \) subject to \( \Ki^\top \wi \leq 0 \)
	is equivalent to the constrained minimization problem,
	\[
		p^*(\lambda) = \min_{w} f(w, \lambda) \quad \text{s.t.} \; \sum_{\bi \in \calB} \norm{\wi}_2 \leq C, \, \Ki^\top \wi \leq 0,
	\]
	where \( C \) is a finite absolute constant.
	Note that this expression is also valid when \( \lambda = 0 \) as the
	min-norm solution to the unregularized least squares problem
	obeys the constraint.

	Thus, \( \solfn \) is a continuous optimization problem over a continuous
	(constant in this case)
	compact constraint set and the classical result of \citet{berge1997topological}
	(see also \citet{hogan1973point}[Theorem 7]) implies
	\( p^* \) is continuous.
\end{proof}

\mapContinuityCGL*
\begin{proof}
	Joint continuity of the objective
	\[
		f(w, \lambda) = \half \norm{X w - y}_2^2 + \lambda \sum_{b_i \in \calB} \norm{\wi}_2,
	\]
	combined with continuity of the (constant) constraint allows us to use
	\citet[Theorem 1]{robinson1974sufficient} to obtain that
	that \( \solfn \) is upper semi-continuous.
	Since \( \solfn \) is convex and bounded, it is compact.
	It is thus also uniformly compact and upper semi-continuity is
	equivalent to closedness \citep{hogan1973point}[Theorem 3].
	We conclude that \( \solfn \) is closed as claimed.

	If \( X \) is full column rank, then the constrained group
	lasso solution is unique for all \( \lambda \geq 0 \).
	The solution map is a singleton on \( \R^+ \), and closedness
	and openness are equivalent properties for singleton maps.
	Since we have already shown it is closed, the solution map
	must also be open.
	An identical argument shows that the solution map is open
	at on any interval over which the solution is unique.

	Now we show the reverse implication by proving the contrapositive.
	Assume \( X \) is not full column-rank and suppose \( \Ki = 0 \)
	for each \( \bi \in \calB \).
	The solution map at \( \lambda = 0 \) is the solution set to
	the least squares problem,
	\[
		\min_w \half \norm{X w - y}_2^2,
	\]
	which is known to be \( \solfn(0) = \cbr{\wmin(0) + z : z \in \Null(X)} \).
	While \( \solfn(0) \) is unbounded, it holds that
	\( \solfn(\lambda) \subset C \) for some bounded \( C \) for
	every \( \lambda > 0 \) (\cref{lemma:bounded-solutions-cgl}).
	As a result, there exist uncountably many solutions in \( \solfn(0) \)
	which are not limit points of solutions in \( \solfn(\lambda_k) \)
	as \( \lambda_k \rightarrow 0 \).
	In other words, \( \solfn(0) \) is not open at \( 0 \).
\end{proof}

\modelFitContinuity*
\begin{proof}
	Consider the dual problem from \cref{lemma:lagrange-dual},
	\begin{equation*}
		\begin{aligned}
			\max_{\eta} \;
			 & - \half (\eta - X^\top y) (X^\top X)^+ (\eta - X^\top y) + \half \norm{y}_2^2                                  \\
			 & \quad \quad \text{s.t.} \quad \eta \in \Row(X), \; \norm{\eta_\bi}_2 \leq \lambda \; \forall \; \bi \in \calB.
		\end{aligned}
	\end{equation*}
	The objective function is a convex quadratic and continuous in \( \eta \).
	The constraint set is
	\begin{align*}
		\calC(\lambda)
		 & = \cbr{ \eta : \eta \in \Row(X), : \norm{\eta_\bi}_2 \leq \lambda \; \forall \; \bi \in \calB}.
	\end{align*}
	Let's show that \( \calC \) is continuous.

	Let \( \lambda_k \geq 0 \), \( \lambda_k \into \bar \lambda \)
	and \( \eta_k \in \calC(\lambda_k) \) such that \( \eta_k \into \bar \eta \).
	Since \( \eta_k \in \Row(X) \), \( \bar \eta \in \Row(X) \).
	Moreover,
	\[
		\norm{[\eta_k]_\bi}_2 \leq \lambda_k,
	\]
	so that taking
	limits on both sides implies \( \norm{\bar \eta_k}_2 \leq \bar \lambda \).
	Thus, \( \bar \eta \in \calC(\bar \lambda) \), showing that \( \calC \)
	is closed.

	\citet[Theorem 12]{hogan1973point} states that \( \calC(\lambda) \)
	is open at \( \bar \lambda \) if for each \( \bi \in \calB \), \( g_\bi(\lambda, \eta) = \norm{\eta_\bi}_2 - \lambda \)
	is continuous on \( \bar \lambda \times \calC(\lambda) \), convex in
	\( \eta \), and for fixed \( \lambda \), and there exists \( \bar \eta \)
	such that such that \( g(\lambda, \bar \eta) < 0  \).

	Let us check these conditions.
	First, observe that \( g_\bi \) is continuous and convex in \( \eta \) for
	any choice of \( \lambda \).
	Taking \( \bar \eta = 0 \), we find
	\[
		g_\bi(\lambda, 0) = -\lambda < 0,
	\]
	as long as \( \lambda > 0 \).
	Thus, \( \calC(\lambda) \) is open at each \( \lambda > 0 \).

	At \( \lambda = 0 \), \( \calC(\lambda) = \cbr{0} \).
	Since \( \calC \) is closed everywhere, we conclude it is open at
	\( \lambda = 0 \).
	Putting these results together proves that \( \calC \) is continuous.

	Recall that the dual solution satisfies \( \eta^*_\bi = \ci \),
	so that it is always unique.
	Combining this fact with \citet[Theorem 1]{robinson1974sufficient}
	implies that \( \lambda \mapsto \eta^*(\lambda) \) is a continuous
	function.
	Since we have
	\[
		\eta^*(\lambda) = c(\lambda) = X^\top (y - X w^*(\lambda)),
	\]
	it must be that the model fit \( \hat y(\lambda) = X w^*(\lambda) \)
	is continuous as well (any discontinuities must be in \( \Null(X) \),
	but \( \hat y \) is orthogonal to \( \Null(X) \)).

	Finally, using \cref{prop:value-continuity-cgl}, we have
	\[
		p^*(\lambda) - \half \norm{\hat y(\lambda) - y}_2^2 = \lambda \sum_{\bi \in \calB} \norm{\wi}_2,
	\]
	is a sum of continuous functions and thus continuous.
	Writing
	\[
		g(\lambda) = \sum_{\bi \in \calB} \norm{\wi}_2 = \sbr{\lambda \sum_{\bi \in \calB} \norm{\wi}_2} \sbr{\frac{1}{\lambda}},
	\]
	as the product of continuous functions shows
	\( g(\lambda) \) is continuous at every \( \lambda > 0 \).
\end{proof}

\subsection{The Min-Norm Path}


\begin{proposition}\label{prop:min-norm-interp}
	Consider the min group-norm interpolation problem,
	\[
		\min_{w} \sum_{\bi \in \calB} \norm{\wi}_2 \quad \text{s.t.} \, \, X w = y.
	\]
	There exist \( X, y \) such that
	the minimum-group-norm solution to this problem
	is unique and not in \( \Row(\Xas) \).
\end{proposition}
\begin{proof}
	We provide a counter-example where the solution is not the row space of
	the active set.
	Choose
	\[
		X =
		\left[
			\begin{array}{cc|c}
				1 & 2 & 0 \\
				1 & 0 & 2 \\
			\end{array}
			\right],
		\quad
		y = \begin{bmatrix}
			1 \\
			1
		\end{bmatrix},
	\]
	where the vertical line indicates the block structure, i.e.,
	\( b_1 = \cbr{1, 2} \) and \( b_2 = \cbr{3} \).

	Clearly a solution using only \( b_2 \) cannot interpolate the data,
	so the active set must be \( \cbr{b_1, b_2} \)  or \( \cbr{b_1} \).
	If the active set is \( b_1 \), then the minimum norm interpolating solution
	can only be \( w = [1 \; 0 \; 0]^\top \), which has group norm \( 1 \).

	Now, consider when the active set is \( \cbr{b_1, b_2} \).
	The interpolating solution in \( \Row(X) \) satisfies the following
	system
	\begin{align*}
		\left[
			\begin{array}{cc|c}
				1 & 2 & 0 \\
				1 & 0 & 2 \\
			\end{array}
			\right]
		\rbr{
			\alpha
			\begin{bmatrix}
				1 \\ 2\\ 0
			\end{bmatrix}
			+
			\beta
			\begin{bmatrix}
				1 \\ 0\\ 2
			\end{bmatrix}
		}
		 & = \begin{bmatrix}
			     1 \\
			     1
		     \end{bmatrix} \\
		\implies
		\begin{bmatrix}
			5 \alpha + \beta \\
			\alpha + 5 \beta
		\end{bmatrix}
		 & =
		\begin{bmatrix}
			1 \\
			1
		\end{bmatrix}.
	\end{align*}
	Solving for \( \alpha \) and \( \beta \) yields
	\( \alpha = 1 - 5 \beta \) and \( 24 \beta = 4 \),
	which implies \( \beta = 1/6 \) and \( \alpha = 1/6 \).
	The optimal \( w^* \) is thus
	\[
		w^* =
		\begin{bmatrix}
			1/3 \\
			1/3 \\
			1/3
		\end{bmatrix},
	\]
	and the group norm of \( w^* \) is
	\[
		\sum_{\bi \in \calB} \norm{\wi^*}_2 = \sqrt{2/9} + 1/3  = (1 + \sqrt{2}) / 3 < 1.
	\]
    Thus, we can discount the interpolating solution using only \( b_1 \).

	Now let's see if we can reduce the norm by including directions in \( \Null(X) \).
	The Null space is orthogonal to both rows of \( X \), from which we conclude
	\[
		\Null(X) =
		\cbr{
			\gamma z : z =
			\begin{bmatrix}
				-2/3 \\
				1/3  \\
				1/3
			\end{bmatrix},
			\gamma \in \R
		}.
	\]
	Any vector \( w' = w^* + \gamma z \) is an interpolating solution,
	so it only remains to check if there is a choice of \( \gamma \)
	that decreases the group norm.
	Assuming \( \gamma > -1 \),
	\begin{align*}
		\sum_{\bi \in \calB} \norm{\wi^*}_2
		 & \leq \sum_{\bi \in \calB} \norm{\wi'}_2                                                                             \\
		\iff \frac{1 + \sqrt{2}}{3}
		 & = \sqrt{(\frac{1}{3} - \frac{2 \gamma}{3})^2 + (\frac{1}{3} + \frac{\gamma}{3})^2}
		+ \abs{\frac{1 + \gamma}{3}}                                                                                           \\
		\iff \frac{1 + \sqrt{2}}{3} - \frac{1 + \gamma}{3}
		 & \leq \sqrt{(\frac{1}{3} - \frac{2 \gamma}{3})^2 + (\frac{1}{3} + \frac{\gamma}{3})^2} \tag{since \( \gamma > -1 \)} \\
		\iff \rbr{\frac{1 + \sqrt{2}}{3} - \frac{1 + \gamma}{3}}^2
		 & \leq (\frac{1}{3} - \frac{2 \gamma}{3})^2 + (\frac{1}{3} + \frac{\gamma}{3})^2                                      \\
		\iff \rbr{\sqrt{2} - \gamma}^2
		 & \leq (1 - 2 \gamma)^2 + (1 + \gamma)^2
	\end{align*}
	The left-hand side satisfies
	\[
		\rbr{\sqrt{2} - \gamma}^2 = 2 - 2\sqrt{2}\gamma + \gamma^2,
	\]
	while the right-hand side is
	\[
		(1 - 2 \gamma)^2 + (1 + \gamma)^2 = 1 - 4 \gamma + 4 \gamma^2 + 1 + 2 \gamma + \gamma^2
		=  2 - 2 \gamma + 5 \gamma^2.
	\]
	As a result,
	\begin{align*}
		\sum_{\bi \in \calB} \norm{\wi^*}_2
		                                                                  & \leq \sum_{\bi \in \calB} \norm{\wi'}_2 \\
		\iff 2 - 2 \gamma + 5 \gamma^2 - 2 - \gamma^2 + 2 \sqrt{2} \gamma & \geq 0                                  \\
		\iff 2 (\sqrt{2} - 1)\gamma + 4 \gamma^2                          & \geq 0.
	\end{align*}
	However, it is easy to check that this fails for \( \gamma \in ((1 - \sqrt{2})/2, 0) \).
	So the minimum-group-norm interpolating solution is not in \( \Row(\Xas) \).

    Now let's show that the minimum-group-norm solution is unique.
    This holds if the group norm as a function of \( \gamma \),
    \[
    g(\gamma) = \sqrt{(\frac{1}{3} - \frac{2 \gamma}{3})^2 + (\frac{1}{3} + \frac{\gamma}{3})^2}
		+ \abs{\frac{1 + \gamma}{3}},
    \]
    is strictly convex. 
    The function \( h_2(\gamma) = \abs{\frac{1 + \gamma}{3}} \) is convex,
    so we need only show that 
    \[
        h_1(\gamma) = \sqrt{(\frac{1}{3} - \frac{2 \gamma}{3})^2 + (\frac{1}{3} + \frac{\gamma}{3})^2}
        = \frac{1}{3} \sqrt{2 - 2\gamma + 5 \gamma^2},
    \]
    is strictly convex.
    Taking the second derivative, we find
    \[
        h_1''(\gamma) = \frac{3}{\rbr{2 - 2\gamma + 5 \gamma^2}^{3/2}} > 0
    \]
    for all \( \gamma \). 
    This completes the proof.
\end{proof}

\begin{lemma}\label{lemma:lambda-convergence}
	Let \( W_{g} \) be the set of least squares solutions with minimum group norm.
	That is,
	\[
		W_{g} = \argmin_{w} \cbr{\sum_{\bi \in \calB} \norm{\wi}_2 : X^\top X w = X^\top y }.
	\]
	Then every limit point of the min-norm group lasso solution
	as \( \lambda \rightarrow 0 \) lies in \( W_g \).
\end{lemma}
\begin{proof}
	Let \( \lambda_k \rightarrow 0 \)
	and observe that \( \wmin(\lambda_k) \) has at least one limit point
	since it is bounded (\cref{lemma:bounded-solutions-cgl}).
	Since \( \norm{\ci(\lambda_k)}_2 \leq \lambda_k \),
	we see that \( \lim_{k} \norm{\ci(\lambda_k)}_2 = 0 \)
	and thus \( \lim_k \ci(\lambda_k) = 0 \).

	FO optimality conditions imply
	\begin{equation}
		(X^\top X) \wmin(\lambda) = X^\top y - c(\lambda),
	\end{equation}
	which, taking limits on both sides,
	gives
	\begin{equation}
		\lim_k (X^\top X) \wmin(\lambda_k) = X^\top y.
	\end{equation}
	That is, every limit point \( \bar w \) of \( \wmin(\lambda_k) \) is
	a least squares solution satisfying \( h(\bar w) \leq h(w_g) \).
	We conclude that \( \bar w \in W_g \).
\end{proof}

\rowCE*
\begin{proof}
	Consider the setting of \cref{prop:min-norm-interp},
	with
	\[
		X =
		\left[
			\begin{array}{cc|c}
				1 & 2 & 0 \\
				1 & 0 & 2 \\
			\end{array}
			\right],
		\quad
		y = \begin{bmatrix}
			1 \\
			1
		\end{bmatrix},
	\]
	where the vertical line indicates the block structure, i.e.,
	\( b_1 = \cbr{1, 2} \) and \( b_2 = \cbr{3} \).
	We have shown that the min group norm interpolant is unique,
	is supported on \( b_1 \) and \( b_2 \), and does not lie
	in \( Row(X) \).
	Let \( w_g \) be this solution.

	Let \( \lambda_k \downarrow 0 \).
	By \cref{lemma:lambda-convergence}, every limit point of
	\( \wmin(\lambda_k) = w_g \).
	Thus, \( \lim_k \wmin(\lambda_k) \) exists and is exactly
	\( w_g \).
	Moreover, \( \wmin(\lambda_k) \) must be supported
	on \( b_1 \) and \( b_2 \) for all \( k \) sufficiently large.

	Decomposing \( w_g = a + b \) and \( \wmin(\lambda_k) = r_k + n_k \)
	where \( a, r_k \in \Row(X) \) and \( b, n_k \in \Null(X) \), we see
	that
	\[
		\norm{w_g - \wmin(\lambda_k)}_2^2
		= \norm{a - r_k}_2^2
		+ \norm{b - n_k}_2^2 \rightarrow 0,
	\]
	implying that \( n_k \neq 0 \) for sufficiently large \( k \).
	In other words, the min-norm solution to the group lasso problem
	fails to fall in \( \Row(X) \) for some \( \lambda > 0 \).
\end{proof}

\minNormProgram*
\begin{proof}
	Let \( w \in \solfn(\lambda) \).
	By \cref{prop:sol-fn-cgl}, \( \wi = \alpha_\bi \vi \)
	where \( \alpha_\bi \geq 0 \).
	Moreover, \( \alpha_\bi = 0 \) for every
	\( \bi \in \calB \setminus \calS_\lambda \).
	As a result,
	\begin{align*}
		\norm{\w}^2_2
		 & = \norm{\we}^2_2                                         \\
		 & = \sum_{\bi \in \calS_\lambda} \norm{\alpha_\bi \vi}^2_2 \\
		 & = \lambda \norm{\alpha}_2^2,
	\end{align*}
	where the last equality follows from \( \bi \in \calS_\lambda \implies \bi \in \equi \),
	which enforces \( \norm{\vi}_2 = \lambda \).

	Now suppose \( \alpha^* \) is optimal for the cone program
	and let \( w \in \R^d \) such that \( \wi = \alpha_\bi^* \vi \).
	By construction, \( \alpha^*_\bi = 0 \)
	for all \( \bi \in \calB \setminus \calS_\lambda \)
	(or it could not be optimal)
	so that \( \wi = 0 \) for all \( \bi \in \calB \setminus \calS_\lambda \).
	Moreover,
	\[
		X w = \sum_{\bi \in \calS_\lambda} \alpha_\bi^* \Xbi \vi = \hat y,
	\]
	Thus, \( w \) solves
	\[
		\begin{aligned}
			\argmin_{w} \norm{w}_2^2
			\, \, \text{ s.t.} \, \,
			 & \forall \, \bi \in \calS_\lambda, \, \wi = \alpha_\bi \vi, \alpha_\bi \geq 0, \, \\
			 & \forall \, b_j \in \calB \setminus \calS_\lambda, \w_{b_j} = 0, \,
			X w = \hat y.
		\end{aligned}
	\]
	Invoking \cref{prop:sol-fn-cgl} now proves that \( w \) is the min-norm solution.
\end{proof}

\begin{restatable}{lemma}{l2Reformulation}\label{lemma:l2-reformulation}
	The \( \ell_2 \)-penalized group lasso problem in \cref{eq:l2-penalized-cgl}
	is equivalent to the following CGL problem:
	\begin{equation}\label{eq:l2-reformulation}
		\begin{aligned}
			\min_{w} \,
			 & \half \norm{\tilde X w - \tilde y}_2^2 + \lambda \sum_{\bi \in \calB} \norm{\wi}_2 \\
			 & \quad \text{s.t.} \quad \Ki^\top \wi \leq 0  \text{ for all } \bi \in \calB.
		\end{aligned}
	\end{equation}
	where we have defined the extended data matrix and targets
	\[
		\tilde X =
		\begin{bmatrix}
			X \\
			\sqrt{\delta} I
		\end{bmatrix}
		\quad \quad
		\tilde y =
		\begin{bmatrix}
			y \\
			0
		\end{bmatrix}.
	\]
    Moreover, \( \tilde X \) is full column-rank and thus the CGL solution is
    unique.
\end{restatable}
\begin{proof}
	It is straightforward to show the equivalence by direct calculation.
	For any \( w \in \R^d \),
	\begin{align*}
		\half \norm{\tilde X \w - \tilde y}_2^2
		 & = \half \norm{X w - y}_2^2 + \half \norm{\sqrt{\delta} I w - 0}_2^2 \\
		 & = \half \norm{X w - y}_2^2 + \frac{\delta}{2} \norm{w}_2^2,
	\end{align*}
	Substituting this identity into \cref{eq:l2-reformulation} establishes the
	equivalence.

	It is clear by inspection that \( \tilde X \) is full column rank.
	Then \( \Null(X_\calC) = \emptyset \) for all \( \calC \subset \calB \)
	and the solution is unique by \cref{prop:sol-fn-cgl}.
\end{proof}

\penConvergence*
\begin{proof}
	First we show that \( \norm{w^\delta(\lambda)}_2 \leq \norm{w^*(\lambda)}_2 \).
	Suppose by way of contradiction that
	\( \norm{w^\delta(\lambda)}_2 > \norm{w^*(\lambda)}_2 \)
	for some \( \delta > 0 \).
	Since
	\begin{align*}
		\half \norm{Xw^*(\lambda) - y}_2^2
		+ \lambda \sum_{\bi \in \calB} \norm{\wi^*(\lambda)}_2
		 & =
		\min_{w : \Ki \wi \leq 0} \half \norm{Xw - y}_2^2
		+ \lambda \sum_{\bi \in \calB} \norm{\wi}_2 \\
		 & \leq
		\half \norm{Xw^\delta(\lambda) - y}_2^2
		+ \lambda \sum_{\bi \in \calB} \norm{\wi^\delta(\lambda)}_2,
	\end{align*}
	we deduce
	\[
		\half \norm{Xw^*(\lambda) - y}_2^2
		+ \lambda \sum_{\bi \in \calB} \norm{\wi^*(\lambda)}_2
		+ \frac{\delta}{2} \norm{w^*(\lambda)}_2^2
		<
		\half \norm{Xw^\delta(\lambda) - y}_2^2
		+ \lambda \sum_{\bi \in \calB} \norm{[w^\delta(\lambda)]_\bi}_2
		+ \frac{\delta}{2} \norm{w^\delta(\lambda)}_2^2,
	\]
	which contradicts optimality of \( w^\delta(\lambda) \).
	So \( \norm{w^\delta(\lambda)}_2 \leq \norm{w^*(\lambda)}_2 \)
	for all \( \delta > 0 \).
	As a result, the sequence
	\( \cbr{w^{\delta_k}(\lambda)}_{\delta_k} \), where \( \delta_k \downarrow 0 \),
	is bounded and admits at least one convergent subsequence.
	Let \( \bar w(\lambda) \) be the limit point associated with one
	such subsequence;
	clearly \( \norm{\bar w(\lambda)}_2 \leq \norm{w(\lambda)^*}_2 \).

	Let's show that \( \bar w(\lambda) \) is a solution to the group lasso
	problem by checking the KKT conditions.
	Suppose \( \lambda > 0 \).
	Stationarity of the Lagrangian is
	\[
		X^\top (Xw^{\delta_k}(\lambda) - y) + K \rho^{\delta_k}(\lambda) + s^{\delta_k}(\lambda) + \delta_k w^{\delta_k}(\lambda) = 0,
	\]
	where \( s_\bi^{\delta_k}(\lambda) \in \partial \lambda \norm{\wi^{\delta_k}}_2 \).
	Since \( \norm{s_\bi^{\delta_k}(\lambda)}_2 \leq \lambda \) and \( w^{\delta_k}(\lambda) \)
	is bounded, clearly \( K \rho^{\delta_k}(\lambda) \) is also bounded.

	Dropping to a subsequence if necessary, let \( \lim_{k} w^{\delta_k}(\lambda) = \bar w \)
	and \( \lim_k \Ki \ri^{\delta_k} = \bar{z}_{\bi} \).
	Define
	\[
		R_{1/n} = \cbr{\ri : \norm{\Ki \ri - \bar{z}_{\bi}}_\infty \leq \frac{1}{n}, \ri \geq 0}.
	\]
	The sequence of sets \( R_{1/n} \) is polyhedral and thus retractive.
	Moreover, for each \( n \in \bbN \), there exists \( k \) such that
	\[ \norm{\Ki \ri^{\delta_k} - \bar{z}_{\bi}}_\infty \leq 1 / n, \]
	since \( \Ki \ri^{\delta_k} \rightarrow \bar{z}_{\bi} \).
	Recalling \( \ri^{\delta_k} \geq 0 \) shows that \( \R_{1/n} \) is non-empty.
	The limit of a sequence of nested, non-empty, retractive sets is also
	non-empty \citep[Proposition 1.4.10]{bertsekas2009convex}.
	Moreover, since the limit is exactly
	\[
		\bar R = \cbr{\ri : \Ki \ri = \bar{z}_{\bi}, \ri \geq 0},
	\]
	we deduce that there exists \( \bar \rho \geq 0 \) such that
	\( \Ki \bar \ri = \bar{z}_{\bi} \).

	Taking limits on either side of the stationarity condition, we find
	\begin{align*}
		\Xbi^\top (y - X\bar w(\lambda)) - \Ki \bar \rho(\lambda) = \bar s_\bi,
	\end{align*}
	where the limit point \( \bar s_\bi \) satisfies \( \norm{\bar s_\bi} \leq \lambda \).
	If \( \bi \in \calB \setminus \act(\bar w(\lambda)) \),
	then \( \bwi \) satisfies stationarity.

	Let \( \bi \in \act(\bar w(\lambda)) \).
	Since \( \w^{\delta_k}(\lambda) \rightarrow \bar w(\lambda) \),
	\( \norm{\wi^{\delta_k} - \bar \wi}_2 \rightarrow 0 \)
	and it must happen that
	\( \wi^{\delta_k} > 0 \) for all \( k \) sufficiently large.
	That is,
	\( \calA(w^{\delta_k}(\lambda)) \supseteq \calA(\bar w) \)
	for all \( k \geq k' \).
	Using \( \bi \in \calA(w^{\delta_k}(\lambda)) \)
	provides a closed-form expression for \( s_\bi^{\delta_k} \):
	\begin{align*}
		\lim_{k} s^{\delta_k}(\lambda)
		 & = \lambda \lim_{k} \frac{\wi^{\delta_k}(\lambda)}{\norm{\wi^{\delta_k}(\lambda)}_2} \\
		 & = \lambda \frac{\bwi(\lambda)}{\norm{\bar \wi(\lambda)}_2},
	\end{align*}
	which shows that \( \bar s_\bi \) is a subgradient of \( \lambda \norm{\wi}_2 \).
	We conclude that the Lagrangian is stationary in \( \bwi \) as well.

	Let us check the remainder of the KKT conditions.
	For feasibility, it is straightforward to observe that
	\[
		\Ki^\top \wi^{\delta_k}(\lambda) \leq 0 \quad \forall \, k \implies \Ki^\top \bwi(\lambda) \leq 0.
	\]
	Similarly,
	\[
		\abr{\ri^{\delta_k}, \Ki^\top \wi^{\delta_k}} = 0 \quad \forall \, k \implies \abr{\bar \ri, \Ki^\top \bwi} = 0,
	\]
	which, combined with \( \bar \rho \geq 0 \), is sufficient to establish
	complementary slackness.
	We have shown the subsequential limits \( (\bar w, \bar \rho) \)
	satisfies the KKT conditions and thus \( \bar w \)
	is a solution to the constrained group lasso problem.

	Since the min-norm solution is unique and
	\( \norm{\bar w(\lambda)}_2 \leq \norm{w^*(\lambda)}_2 \),
	it must be that \( \bar w(\lambda) = w^*(\lambda) \).
	Noting that this holds for every limit point implies
	\( \lim_{\delta \downarrow 0} w^{\delta}(\lambda) \) exists and is \( w^*(\lambda) \).
	This completes the proof for \( \lambda > 0 \).

	If \( \lambda = 0 \), then the proposition follows similarly with
	the additional observation that \( s^{\delta_k}(0) = 0 \)
	for all \( k \).
\end{proof}

\subsection{Sensitivity}


\constraintReduced*
\begin{proof}
	Let \( w \) be as in the theorem statement.
	We starting by showing that \( w \) obtains the optimal objective value
	for the reduced problem:
	\begin{align*}
		\half \norm{\Xa \w_\act - y}_2^2 + \lambda \sum_{\bi \in \act} \norm{\wi}_2
		 & =
		\min_{w : \Ki \wi \leq 0} \,
		\half \norm{X w - y}_2^2 + \lambda \sum_{\bi \in \calB} \norm{\wi}_2                                   \\
		 & \leq
		\min_{\wa : \Ka \wa \leq 0} \, \half \norm{\Xa \wa - y}_2^2 + \lambda \sum_{\bi \in \act} \norm{\wi}_2 \\
		 & \leq \half \norm{\Xa \w_\act - y}_2^2 + \lambda \sum_{\bi \in \act} \norm{\wi}_2,
	\end{align*}
	where the last inequality makes explicit use of feasibility of \( \w_\act \).
	Since \( \w_\act \) is feasible for the reduced problem and attains the
	minimum objective value, it must be optimal.
	Note that it is straightforward to check that the active blocks of the min-norm
	dual parameter \( \rmin_\act \) are dual optimal for the reduced problem.

	Now, let \( \wa' \) be an optimal solution to the reduced problem.
	Since \( \rmin_\act \) is dual optimal for the reduced problem,
	\cref{prop:sol-fn-cgl} implies
	\[
		\wi' = \alpha_\bi' \vi,
	\]
	for every \( \bi \in \act \), with \( \alpha_\bi' \geq 0 \).
	Since \( \Xa \wa' = \Xa \wa \), we deduce
	\[
		\sum_{\bi \in \act} (\alpha_\bi - \alpha_\bi')\Xbi \bi = 0,
	\]
	which contradicts minimality of \( \w \) unless \( \alpha_\bi = \alpha_\bi' \).
	That is, \( \wa' = \wa \).
	We conclude that the reduced problem provides the unique minimal solution
	with support \( \act \).
\end{proof}

\begin{lemma}\label{lemma:second-order-stationary}
	Let \( w \in \solfn(\lambda) \) be minimal.
	Then \( w \) is a second-order stationary point of the reduced problem
	(\cref{eq:constraint-reduced}).
\end{lemma}
\begin{proof}
	Define \( M(w) \) to be the block-diagonal projection matrix given by
	\begin{equation}\label{eq:projection-matrix-block}
		M(w)_\bi
		= \frac{1}{\norm{\wi}_2} \rbr{I - \frac{\wi}{\norm{\wi}_2}\frac{\wi^\top}{\norm{\wi}_2}}.
	\end{equation}
	The Hessian of the Lagrangian of the reduced problem with respect to
	\( w \) is exactly
	\[
		\nabla^2_w \calL(\wa, \ra) =
		\Xa^\top \Xa + \lambda M_\act(w).
	\]
	A sufficient condition for \( w \) to be second-order stationary is
	that this Hessian is positive-definite.
	We now shows this fact holds.

	Clearly \( \nabla^2_w \calL(\wa, \ra) \) is positive semi-definite
	as it is the sum of a PSD projection matrix
	and a Gram matrix, which is always PSD.
	Let \( \bar w \in \R^{|\act|} \) such that \( \bar w \neq 0 \).
	Suppose that
	\begin{align*}
		0
		 & = \bar w^\top \nabla^2_w \calL(\wa, \ra) \bar w                    \\
		 & = \bar w^\top \Xa^\top \Xa \bar w + \bar w^\top\lambda M(w) \bar w \\
		 & = \norm{\Xa w}_2^2 + \lambda w^\top M(w) w.
	\end{align*}
	Since \( M(\bar w) \) is PSD,
	it must hold that
	\[
		\bar w^\top M(w) \bar w = 0,
	\]
	which is true if and only if \( \bwi = \beta_\bi \wi \),
	\( \beta_\bi \in \R \), for each \( \bi \in \act \).
	As a result, we find that
	\[
		\Xa w - \Xa \bar w = \sum_{\bi \in \act} (1 - \beta_\bi) \alpha_{\bi} \Xbi \vi = 0,
	\]
	which contradicts minimality of \( \bar w \) unless \( \beta_{\bi} = 1 \)
    for each \( \bi \in \act \).
	But then \( \norm{\Xa w}_2^2 > 0 \) and
	we conclude that the Hessian is positive-definite as desired.
\end{proof}

\localSolFn*
\begin{proof}
	Recall from \cref{prop:constraint-reduced} that \( \wa \) is the unique
	solution to the reduced group lasso problem.
	In fact, as we show in \cref{lemma:second-order-stationary}, \( \wa \) is
	a second order stationary point for the reduced problem.
	Now, combining this fact with LICQ and SCS
	and using standard results on differential sensitivity from optimization
	theory (see, e.g. \citet[Theorem 5.1]{fiacco1990sensitivity} and
	the references therein)
	we obtain the following:

	For \( (\lambda, y) \) in a neighborhood of \( \bar \lambda, \bar y \),
	there exists a unique once continuously differentiable function
	\[
		\tilde l(\lambda, y) =
		\begin{bmatrix}
			\tilde h(\lambda, y) \\
			\tilde g(\lambda, y),
		\end{bmatrix}
	\]
	such that \( \tilde h(\bar \lambda, \bar y) = \wa \),
	\( \tilde g(\bar \lambda, \bar y) = \rmin_\act \),
	and \( \tilde l(\lambda, y) \) is a primal-dual solution to the reduced
	problem.

	Now we show that \( \tilde l \) can be extended from the reduced problem
	to obtain a local solution function for the constrained group lasso.
	Define \( h(\lambda, y) \) such that
	\( h_\act(\lambda, y) = \tilde h(\lambda, y) \)
	and \( h_{\calB \setminus \act}(\lambda, y) = 0 \).
	We shall show how to extend \( g \) shortly.
	For \( \bi \in \act \), the pair
	\( h_{\bi}(\lambda, y), g_\bi(\lambda, y) \)
	verifies the KKT conditions (which are separable over block)
	since it verifies them for the reduced problem.
	So, we need only consider \( \bi \in \calB \setminus \act \).

	First, consider \( \bi \in \calB \setminus \equi \).
	In this case, we have
	\[
		\norm{\Xbi^\top (\bar y - X h(\bar \lambda, \bar y)) + \Ki \ri(\bar \lambda, \bar y)}_2
		< \bar \lambda,
	\]
	Since this inequality is strict and
	\[
		z(\bar\lambda, \bar y) = \Xbi^\top (\bar y - X h(\bar \lambda, \bar y)),
	\]
	is continuous in \( \bar \lambda, \bar y \),
	there exists a neighborhood of \( \bar \lambda, \bar y \)
	on which
	\[
		\norm{z(\lambda, y) + \Ki \ri(\bar \lambda, \bar y)}_2
		\leq \lambda.
	\]
	Since \( \ri(\bar \lambda, \bar y) \geq 0 \) and \( \wi = 0 \),
	dual feasibility and complementary slackness hold.
	We conclude that the extension
	\( g_\bi(\lambda, y) = \ri(\bar \lambda, \bar y) \)
	satisfies KKT conditions on this neighbourhood.

	Now suppose \( \bi \in \equi(\bar \lambda, \bar y) \setminus \act \).
	If
	\[
		\norm{\Xbi^\top (y - X h(\lambda, y))}_2 = \lambda,
	\]
	then taking \( g_\bi(y, \lambda) = 0 \) satisfies KKT conditions.
	Otherwise, observe that
	\[
		\norm{\Xbi^\top (\bar y - X h(\bar \lambda, \bar y)) + \Ki \ri(\bar \lambda, \bar y)}_2
		= \bar \lambda,
	\]
	must hold for some dual parameter \( \ri(\bar \lambda, \bar y) \)
	by KKT conditions.
	Moreover, SCS implies that we can choose the dual parameter to satisfy,
	\[ \ri(\bar \lambda, \bar y) > 0, \]
	since \( \Ki^\top \wi(\bar \lambda, \bar y) = 0 \).
	Finally, because
	\[ \Xbi^\top (y - X h(\lambda, y)) \]
	is a continuous function of \( (y, \lambda) \), taking
	\( \lambda, y \) sufficiently close to \( \bar \lambda, \bar y \)
	implies there exists \( \ri(\lambda, y) \geq 0 \) such that
	\[
		\norm{\Xbi^\top (y - X h(\lambda, y)) + \Ki \ri(\lambda, y)}_2
		\leq \lambda.
	\]
	Now we choose our extension to be
	\( g_\bi(\lambda, y) = \ri(\lambda, y) \)
	so that \( (h_\bi, g_\bi) \) satisfies stationarity of the Lagrangian
	as well.
	Since \( g_\bi \) is feasible and \( h_\bi \)  is the zero function,
	primal feasibility, dual feasibility, and complementary slackness also hold.

	Since \( l = (g, h) \) satisfies the KKT conditions in a local neighborhood
	of \( \bar \lambda, \bar y \), it is exactly a local solution function.
	Moreover, since \( g_{\calB \setminus \calA}(\lambda, y) = 0 \)
	over this neighborhood, it is easy to see that the gradient
	for parameter blocks in \( \calB \setminus \calA \)
	is \( 0 \).
	For \( g_\act \), \citet[Theorem 5.1]{fiacco1990sensitivity} implies that the
	gradients are given as follows:

	Recall from \cref{lemma:second-order-stationary},
	that \( M_\act \) is a block-diagonal projection matrix
	with blocks given by
	\[
		M(\bar w)_\bi = \frac{1}{\norm{\bwi}_2} \rbr{I - \frac{\bwi}{\norm{\bwi}_2}\frac{\bwi^\top}{\norm{\bwi}_2}}.
	\]
	Then, the Jacobian of \( \nabla_w \calL(\bar \wa, \bar \ra) \)
	for the reduced problem with respect to the primal-dual parameters is
	given by
	\[
		D =
		\begin{bmatrix}
			\Xa^\top \Xa + M(\bar w) & \Ka                            \\
			\bar \ra \odot \Ka       & \text{diag}(\Ka^\top \bar \wa)
		\end{bmatrix}.
	\]
	It also holds that \( D \) is invertible.
	Finally, let \( u_i = \frac{w_i}{\norm{\wi}_2} \)
	and \( u \) the concatenation of these vectors.
	We are now able to write the
	Jacobians of \( w(y, \lambda) \)
	with respect to \( y \) and \( \lambda \)
	as follows:
	\[
		\nabla_\lambda w(\bar \lambda, \bar y)
		= - [D^{-1}]_\act u_\act
		\quad \quad
		\nabla_y w(\bar \lambda, \bar y)
		= [D^{-1}]_\act \Xa^\top,
	\]
	where \( [D_\act^{-1}]_\act \) is the \( |\act| \times |\act| \)
	dimensional leading principle submatrix of \( D \).
\end{proof}

\section{Specialization: Proofs}\label{app:specializations}


\begin{lemma}\label{lemma:scaling-unique}
	Let \( \rbr{W_{1}, w_{2}} \) and \( \rbr{W_{1}', w_{2}'} \)
	be two solutions to the non-convex ReLU training problem.
	If for every \( i \in [m]  \),
	it holds that
	\[
		W_{1i} w_{2i} = W_{1i}' w_{2i}',
	\]
	and \( \text{sign}(w_{2i}) = \text{sign}(w_{2i}') \),
	then \( W_1 = W_1' \) and \( w_2 = w_2' \).
	That is, the solutions are the same.
\end{lemma}
\begin{proof}
	The ReLU prediction function \( f_{W_1, w_2} \) is invariant to
	scalings of the form
	\[
		\bar W_{1i} = \alpha W_{1i} \quad \bar w_{2i} = w_{2i} / \alpha,
	\]
	where \( \alpha > 0 \).
	Using this, we deduce that both solutions must satisfy the following
	equations:
	\begin{align*}
		1 & = \argmin_{\alpha} \alpha^2 \norm{W_{1i}}_2^2 + \norm{w_{2i}}_2^2 / \alpha^2    \\
		1 & = \argmin_{\alpha} \alpha^2 \norm{W_{1i}'}_2^2 + \norm{w_{2i}'}_2^2 / \alpha^2,
	\end{align*}
	which in turn implies that
	\[
		\norm{W_{1i}}_2 = \norm{w_{2i}}_2, \quad \quad \norm{W_{1i}'}_2 = \norm{w_{2i}'}_2
	\]
	We deduce
	\begin{align*}
		\norm{W_{1i}}^2_2
		 & = \norm{W_{1i}}_2 \norm{w_{2i}}_2    \\
		 & = \norm{W_{1i} w_{2i}}_2             \\
		 & = \norm{W_{1i}' w_{2i}'}_2           \\
		 & = \norm{W_{1i}' }_2 \norm{w_{2i}'}_2 \\
		 & = \norm{W_{1i}'}^2_2,
	\end{align*}
	where we have used the fact that \( \norm{w_{2i}}_2 = |w_{2i}| \).
	This implies \( W_{1i} = W_{1i}' \) since \( W_{1i} \) and \( W_{1i}' \)
	are collinear, meaning \( w_{2i} = w_{2i}' \) must also hold.
\end{proof}

\begin{lemma}\label{lemma:mapping-equality}
	Let \( \rbr{W_{1}, w_{2}} \) and \( \rbr{\tilde W_{1}, \tilde w_{2}} \)
	be two solutions to the non-convex ReLU training problem.
	If \( \rbr{W_{1}, w_{2}} \) and \( \rbr{\tilde W_{1}, \tilde w_{2}'} \) map
	to the same solution in the convex reformulation, then they are equal up to
	permutations or splits/merges of of neurons with collinear weights.
\end{lemma}
\begin{proof}
	Let \( (W_1, w_2) \) be a solution to the non-convex ReLU training
	problem.
	For each neuron, associate \( W_{1i} \) with the sparsest
	activation pattern \( D_i \) to which \( W_{1i} \) conforms.
	That is, given \( W_{1i} \), we associate it with the following pattern:
	\[
		D_i = \argmin \cbr{ \text{nnz}(D_j) : D_j \in \calD_Z, 2 (D_j - I) Z W_{1i} \geq 0 }.
	\]
	If \( W_{1i} \) and \( W_{1j} \) are associated with the same pattern
	\( D_i \) and \( \text{sign}(w_{2i}) = \text{sign}(w_{2j}) \),
	then \citet{mishkin2022convex} prove \( W_{1i} \) and \( W_{1j} \)
	are collinear, i.e. \( W_{1i} = \alpha W_{1j} \) for some \( \alpha > 0 \).
	They also prove that the model obtained by merging these collinear neurons
	as
	\[
		W'_{1i} = \frac{W_{1i}|w_{2i}| + W_{1j}|w_{2j}|}{\norm{W_{1i}|w_{2i}| + W_{1j}|w_{2j}|}_2^{1/2}},
		\quad \quad
		w'_{2i} = \text{sign}(w_{2i})
		\norm{W_{1i}|w_{2i}| + W_{1j}|w_{2j}|}_2^{1/2},
	\]
	where \( W'_{1j} = 0 \), and \( w'_{2j} = 0 \),
	is also an optimal solution to the non-convex ReLU problem.
	We term this the merge/split symmetry.

	By recursively merging all collinear neurons assigned to the same
	activation pattern for which \( \text{sign}(w_{2i}) =
	\text{sign}(w_{2j}) \) holds, we eventually obtain an optimal
	\emph{reduced} model \( (W_1', w_2') \) such that
	no two non-zero neurons share an activation pattern \( D_i \) and have
	second-layer weights with the same sign.
	We may then place an ordering on the neurons by ordering the \( D_i \)
	matrices from \( 1 \) to \( p \).
	This removes all permutation symmetries.

	Since every signed neuron in the non-convex model now
	corresponds to a unique pattern \( D_i \), the mapping to the convex
	program is given by
	\[
		\begin{aligned}
			v_{i} & =
			\begin{cases}
				W'_{1i} w'_{2i} & \mbox{if \( w'_{2i} \geq 0 \)} \\
				0               & \mbox{otherwise}
			\end{cases} \\
			u_{i} & =
			\begin{cases}
				-W'_{1i} w'_{2i} & \mbox{if \( w'_{2i} < 0 \)} \\
				0                & \mbox{otherwise},
			\end{cases}
		\end{aligned}
	\]
	Applying this procedure to \( \rbr{W_{1}, w_{2}} \) and \( \rbr{\tilde W_{1},
		\tilde w_{2}} \) and recalling that they map to the same solution implies
	\[
		W'_{1i} w'_{2i} = \tilde W'_{1i} \tilde w'_{2i},
	\]
	and \( \text{sign}(w'_{2i}) = \text{sign}(\tilde w'_{2i}) \)
	for every \( i \).
	Invoking \cref{lemma:scaling-unique} is sufficient to prove \( \rbr{W'_{1},
	w'_{2}} \) and \( \rbr{\tilde W_{1}', \tilde w_{2}'} \) are identical.

	To obtain this result, we (i) merged all collinear neurons with the same
	sign in the second layer and then (ii) sorted the neurons according to
	their activation patterns.
	Thus, it must be that \( \rbr{W_1, w_2} \) and \( \rbr{\tilde W_1, \tilde w_2} \)
	differ only by (i) splitting of neurons into collinear neurons and (ii)
	the ordering (i.e. permutation) of the neurons.
	This completes the proof.
\end{proof}

\reluSolFn*
\begin{proof}
	Let the proposed optimal set be the completion of
	\[
		\begin{aligned}
			\calU_\lambda =
			 & \big\{
			(W_1, w_2) :
			\, f_{W_1, w_2}(Z) =  \hat y,
			W_{1i} = (\sfrac{\alpha_{i}}{\lambda})^{\sfrac{1}{2}} v_i,        \\
			 & \hspace{1em} w_{2i} = \xi_i (\alpha_i \lambda)^{\sfrac{1}{2}},
			\alpha_i \geq 0, \, i \in [2p] \setminus \calS_\lambda \Rightarrow \alpha_i = 0
			\big\},
		\end{aligned}
	\]
	over all permutations of the neurons and splits/merges of collinear
	neurons.
	Let \( (W_1, w_2) \in \calU_\lambda \).
	Inverting \cref{eq:convex-to-relu}, we compute a candidate solution \( u \)
	to the convex reformulation as
	\[
		u_i = W_{1i} w_{2i} = \alpha_{i} v_i,
	\]
	for some \( \alpha_i \geq 0 \), where we assumed without loss of generality
	that \( w_{2i} \geq 0 \) (if \( \xi_i = -1 \), then we simply map to a
	``negative'' neuron in the convex reformulation).
	Moreover, \( \alpha_i = 0 \) if \( i \in [2p] \setminus \calS_\lambda \).
	Since the solution mapping preserves the prediction of the neural network,
	we also have \( X u = \hat y \), which is enough to guarantee \( u \)
	is a solution to the convex reformulation using \cref{prop:sol-fn-cgl}.
	Since \( u \) is a solution, we conclude that \( (W_1, w_2) \) must
	be optimal for the non-convex ReLU problem.

	For the reverse inclusion, let \( (W_1, w_2) \in \calO_\lambda \) and
	suppose no permutation or merge/split symmetry of \( (W_1, w_2) \)
	is in \( \calU_\lambda \).
	While \( (W_1, w_2) \not \in \calU_\lambda \), it still maps to
	some solution to the convex reformulation, call it \( u \).
	As \( \calU_\lambda \) is obtained by mapping every solution to the convex
	reformulation (i.e. every \( u \in \solfn_\lambda \) ) to the non-convex
	parameterization, it must be that \( u \) also maps to some
	\( (\tilde W_1, \tilde w_2) \) in \( \calU_\lambda \).
	\cref{lemma:mapping-equality} now shows that \( (W_1, w_2) \) and
	\( (\tilde W_1, \tilde w_2) \) are identical up to permutations and
	splits/merges of collinear neurons
	But, this implies \( (W_1, w_2) \in \calU_\lambda \) up to permutation
	and merge/split symmetries, which is a contradiction.
	We conclude that the optimal set is the completion of \( \calU_\lambda \) as
	claimed.
\end{proof}

\stationaryPoints*
\begin{proof}
	The proof is almost immediate.

	Given any sub-sampled set of activation patterns \( \tilde \calD \subset \calD_Z \),
	\citet[Theorem 3]{wang2021hidden} prove that the solutions to the sub-sampled
	convex program are Clarke stationary points \citep{clarke1990optimization} of
	the non-convex ReLU optimization problem in \cref{eq:non-convex-relu-mlp},
	and vice-versa.
	Using the expression for the CGL solution set in \cref{prop:sol-fn-cgl},
	which applies to sub-sampled convex reformulations as well as the full
	program, we obtain a version of \cref{cor:relu-sol-fn} for stationary
	points.
	That is, every model \( (W_1, w_2) \) in
	\begin{equation*}
		\begin{aligned}
			\calC_\lambda(\tilde \calD) = \big\{
			 & (W_1, \! w_2) :
			\, f_{W_1, w_2}(Z) \! = \! \hat y_{\tilde \calD},
			W_{1i} = (\sfrac{\alpha_{i}}{\lambda})^{\sfrac{1}{2}} v_i(\tilde \calD), \,
			w_{2i} = \xi_i (\alpha_i \lambda)^{\sfrac{1}{2}},                                           \\
			 & \hspace{4em} \alpha_i \geq 0, \, i \in [m] \setminus \calS_\lambda \implies \alpha_i = 0
			\big\},
		\end{aligned}
	\end{equation*}
	is a stationary point of the non-convex ReLU program.
	Taking the union over all sub-sampled sets of activation patterns
	gives \( \calC_\lambda \), which is guaranteed to contain every stationary
	point of the non-convex objective.
	This completes the proof.
\end{proof}

\pUnique*
\begin{proof}
	Since there is only one solution to the convex reformulation, all solutions to the
	non-convex training problem must map to that solution.
	\cref{lemma:mapping-equality} now implies that the solution map for the
	non-convex problem is p-unique, i.e. unique up to permutations and
	splits/merges of collinear neurons.
\end{proof}

\reluUnique*
\begin{proof}
	We assume without loss of generality that only indices from \( 1 \)
	to \( p \) are in \( \equi \).
	By \cref{lemma:unique-cgl}, the constrained group lasso
	admits a unique solution
	if and only if
	\[
		\bigcup_{i \in \equi} \cbr{D_i Z \vi},
	\]
	are linearly independent.
	We now show that this fact holds under the proposed sufficient condition
	by proving the stronger fact that
	\( \bigcup_{i \in \equi} \cbr{[D_i Z]_j : j \in [d]} \)
	are linearly independent with probability one, where
	\( [D_i Z]_j \) is the \( j^{\text{th}} \)
	column of \( D_i Z \).
	Since \( \nnz(D_i) \geq d * p \) and \( Z \) has a continuous
	probability distribution, it holds that
	\( [D_i Z]_j \) has at least
	\( d * p \) non-zero entries with probability 1.
	Let
	\[
		\calS_{ij}
		= \text{Span}\rbr{\bigcup_{i \in \equi} \cbr{[D_i Z]_j : j \in [d]}
			\setminus [D_i Z]_j},
	\]
	and observe that \( \text{dim}(\calS_{ij}) \leq d * p - 1 \).
	As a result, the conditional probability \( [D_i Z]_j \) falls
	in this subspace satisfies
	\[
		\Pr([D_i Z]_j \in \calS_{ij} | \bigcup_{i \in \equi} \cbr{[D_i Z]_j : j \in [d]}
		\setminus [D_i Z]_j) = 0.
	\]
	Taking expectations over the remaining vectors in \( Z \) implies
	\[
		\Pr([D_i Z]_j \in \calS_{ij}) = 0.
	\]
	Finally, using a union bound over \( i, j \) implies that
	\( \bigcup_{i \in \equi} \cbr{[D_i Z]_j : j \in [d]} \) are linearly independent
	almost surely.
\end{proof}

\minimalComplexity*
\begin{proof}
	The proof follows directly from existing results.

	Recall from \citet{pilanci2020convexnn} that there are at most
	\[
		p \in O\rbr{r \frac{n}{r}^{3r}},
	\]
	activation patterns in the convex reformulation and that the complexity
	of computing an optimal ReLU model using a standard interior-point solver
	is \( O(d^3 r^3 (n / r)^{3r}) \).

	We know from \cref{prop:pruning-correctness} that the complexity of
	pruning an optimal neural network with at most \( 2p \) neuron
	is \( O(n^3 p + nd) \).
	Combining these complexities, we find that the cost of optimization dominates
	and overall complexity of computing an optimal and minimal neural network
	grows as \( O(d^3 r^3 (n / r)^{3r}) \).

	Finally, the bound on the number of active neurons follows from the
	fact that \( \dim \Span(\cbr{(X W^*_{1i})_+}_i) \leq n \).
	This completes the proof.
\end{proof}

\minNormNN*
\begin{proof}
	Let \( (u, v) \) be an optimal solution to the convex reformulation.
	\citet{pilanci2020convexnn} show that an optimal solution to the original
	two-layer ReLU optimization problem is given by setting
	\[
		W_{1i} = \frac{u_i}{\sqrt{\norm{u_i}_2}}, w_{2i} = \sqrt{\norm{u_i}_2},
	\]
	and
	\[
		W_{1j} = \frac{v_i}{\sqrt{\norm{v_i}_2}}, w_{2j} = -\sqrt{\norm{v_i}_2},
	\]
	where we define \( \frac{0}{0} = 0 \).
	Then, the \( r \)-value of any such solution can be calculated as
	\begin{align*}
		r(W_1, w_2)
		 & = \sum_{i = 1}^m \norm{W_{1i}}_2^4 + \norm{w_{2i}}_2^4                                                  \\
		 & = \sum_{u_i \neq 0} \norm{\frac{u_i}{\sqrt{\norm{u_i}_2}}}_2^4 + \norm{\sqrt{\norm{u_i}_2}}_2^4
		+ \sum_{v_i \neq 0} \norm{\frac{v_i}{\sqrt{\norm{v_i}_2}}}_2^4 + \norm{\sqrt{\norm{v_i}_2}}_2^4            \\
		 & = \sum_{u_i \neq 0} \norm{u_i}_2^2 + \norm{u_i}_2^2 + \sum_{v_i \neq 0} \norm{v_i}_2^2 + \norm{v_i}^2_2 \\
		 & = 2 \norm{u}_2^2 + 2 \norm{v}_2^2.
	\end{align*}
	That is, \( r(W_{1}, w_2) \) is a monotone transformation of the Euclidean
	norm of \( \rbr{u, v} \).
	Moreover, all optimal ReLU networks up to permutation and neuron-split
	symmetries can be obtained by applying this mapping to the solution to a convex reformulation.
	We conclude that the minimum \( r \)-valued optimal ReLU network before any
	such symmetries  are applied is given by the minimum \( \ell_2 \)-norm
	solution to the convex reformulation.

	Now, clearly permutation symmetries have no affect on the \( r \) value.
	However, splitting neurons into multiple collinear neurons can, in fact,
	reduce \( r \).
	For example, consider splitting a neuron \( (W_{1i}, w_{2i}) \) into
	\( (W_{1i}', w_{2i}') \) and \( (W_{1j}', w_{2j}') \) such that
	\( W_{1i}' \) and \( W_{1j}' \) are collinear, \( \text{sign}(w_{2i}') =
	\text{sign}(w_{2j}') \), and
	\[
		W_{1i} = \frac{W_{1i}' w_{2i}' + W_{1j}' w_{2j}'}{\norm{W_{1i}' w_{2i}' + W_{1j}' w_{2j}'}_2^{1/2}},
		\quad \quad
		w_{2i} = \text{sign}(w_{2i})\norm{W_{1i}' w_{2i}' + W_{1j}' w_{2j}'}_2^{1/2},
	\]
	and \( \norm{W_{1i}'}_2 = \norm{w_{2i}'}_2 \),
	\( \norm{W_{1j}'}_2 = \norm{w_{2j}'}_2 \).
	This operation is known to preserve optimality \citep{mishkin2022convex}.
	However, the conditions on the split neurons also imply
	\begin{align*}
		\norm{W_{1i}}_2^4 + \norm{w_{2i}}_2^4
		 & = 2 \norm{W_{1i}' w_{2i}' + W_{1j}' w_{2j}'}_2^{2}                                   \\
		 & = 2 \norm{W_{1i}' w_{2i}'}_2^2 + 2 \norm{W_{1j}' w_{2j}'}_2^2
		+ 4 \abr{W_{1i}' w_{2i}', W_{1j}' w_{2j}'}                                              \\
		 & > 2 \norm{W_{1i}' w_{2i}'}_2^2 + 2 \norm{W_{1j}' w_{2j}'}_2^2                        \\
		 & = \norm{W_{1i}'}_2^4 + \norm{w_{2i}'}_2^4 + \norm{W_{1j}'}_2^4 + \norm{w_{2j}'}_2^4,
	\end{align*}
	where have used positive collinearity of \( W_{1i}' \) and \( W_{2i}' \).
	As a result, the \( r \)-value of a solution can be decreased by splitting
	neurons.
	Thus, the minimum-norm solution to the convex reformulation is only the
	minimum \( r \)-value optimal ReLU network out of all
	networks which admit no neuron-merge symmetries.
\end{proof}

\oneNeuron*
\begin{proof}
	\citet{mishkin2022convex} show that \autoref{eq:non-convex-relu-mlp} has the same global optimal values as
	\begin{equation}\label{eq:mip}
		\min_{v, \gamma \in \cbr{-1, 1}} \half \norm{\sum_{i=1}^p (X v_i)_+ \gamma - y} + \lambda \sum_{i = 1}^p \norm{w_i}_2.
	\end{equation}
	Moreover, the objective-preserving mapping
	\( W_{1i} = v_i / \sqrt{\norm{v_i}_2} \),
	\( w_{2i} = \gamma_i \sqrt{\norm{v_i}_2} \)
	can be used to obtain an optimal ReLU network from a solution to \cref{eq:mip}.
	We proceed by analyzing this equivalent problem and then use the
	mapping to return to the original non-convex formulation.

	Consider the case \( p = 1 \).
	Let \( X, y \) consist of two training points, \( (x_1, y_1) = (-100, 1) \) and \( (x_2, y_2) = (1, 10) \).
	In what follows, we drop the subscript for \( v \) and \( \gamma \) since \( p = 1 \) and we consider a one-dimensional.
	The optimization problem of interest is
	\begin{equation*}
		\min_{v, \gamma} \half \rbr{(x_1 v)_+ \gamma - y_1}^2 + \rbr{(x_2 v)_+ \gamma - y_2}^2 + \lambda \abs{v}.
	\end{equation*}
	Since \( x_1 < 0 \) and \( x_2 > 0 \), we can re-write this optimization problem as
	\[
		\min_{v, \gamma} \half \mathbbm{1}_{v \leq 0} \rbr{\rbr{x_1 v \gamma - y_1}^2 + y_2^2} + \mathbbm{1}_{v > 0}\rbr{\rbr{x_2 v \gamma - y_2}^2 + y_1^2} + \lambda \abs{v}.
	\]
	By inspection, we see that \( \gamma = +1 \) is optimal in both cases, leading to the following simplified expression:
	\[
		\min_{v} \half \mathbbm{1}_{v \leq 0} \rbr{\rbr{x_1 v - y_1}^2 + y_2^2} + \mathbbm{1}_{v > 0}\rbr{\rbr{x_2 v - y_2}^2 + y_1^2} + \lambda \abs{v}.
	\]
	This is a piece-wise continuous (but non-smooth) quadratic with a breakpoint at \( v^* = 0 \).
	We determine the solution to this minimization problem via a case analysis.

	\textbf{Case 1}: \( v^* = 0 \).
	Then, the optimal objective is trivially \( f(v^*) = y_1^2 + y_2^2 = 101 \).

	\textbf{Case 2}: \( v^* < 0 \).
	First order optimality conditions are
	\[
		x_1 \rbr{x_1 v_{-}^* - y_1} - \lambda = 0 \implies v_{-}^* = \frac{y_1 \cdot x_1 + \lambda}{x_1^2},
	\]
	which is valid only if \( \lambda < |y_1 \cdot x_1| = 100 \).
	The minimum objective value is then
	\begin{align*}
		\half \rbr{\rbr{x_1 v_{-}^* - y_1}^2 + y_2^2} - \lambda v_{-}^*
		 & = \rbr{\rbr{x_1 \frac{y_1 \cdot x_1 + \lambda}{x_1^2} - y_1}^2 + y_2^2} - \lambda \rbr{\frac{y_1 \cdot x_1 + \lambda}{x_1^2}} \\
		 & = \frac{\lambda^2}{x_1^2} + y_2^2 - \lambda \rbr{\frac{y_1 \cdot x_1 + \lambda}{x_1^2}}                                       \\
		 & = -\frac{\lambda y_1}{x_1} + y_2^2                                                                                            \\
		 & = \frac{\lambda}{100} + 100.
	\end{align*}

	\textbf{Case 3}: \( v_{+}^* > 0 \).
	Similarly to the previous case, we obtain
	\[
		x_2 \rbr{x_2 v_{+}^* - y_2} + \lambda = 0 \implies v_{+}^* = \frac{y_2 \cdot x_2 - \lambda}{x_2^2},
	\]
	which is valid only if \( \lambda < |y_2 \cdot x_2| = 10 \).
	In this case, the minimum objective is
	\begin{align*}
		\half \rbr{\rbr{x_2 v_{+}^* - y_2}^2 + y_1^2} + \lambda v_{+}^*
		 & = \rbr{\rbr{x_2 \frac{y_2 \cdot x_2 + \lambda}{x_2^2} - y_2}^2 + y_1^2} - \lambda \rbr{\frac{y_2 \cdot x_2 + \lambda}{x_2^2}} \\
		 & = \frac{\lambda^2}{x_2^2} + y_2^2 + \lambda \rbr{\frac{y_2 \cdot x_2 - \lambda}{x_2^2}}                                       \\
		 & = \frac{\lambda y_2}{x_2} + y_1^2                                                                                             \\
		 & = 10 \lambda + 1.
	\end{align*}

	To combine the cases, observe that
	\begin{align*}
		f(v_+^*) - f(v_-^*)
		 & = \frac{\lambda y_2}{x_2} + y_1^2 - \sbr{-\frac{\lambda y_1}{x_1} + y_2^2} \\
		 & = \lambda \rbr{\frac{x_1 y_2 + x_2 y_1}{x_1 x_2}} + y_1^2 - y_2^2          \\
		 & = \lambda \rbr{\frac{-100 (10) + 1 (1)}{-100 (1)}} + 1 - 100               \\
		 & = 9.99 \lambda - 99.
	\end{align*}
	We deduce that \( f(v_+^*) - f(v_-^*) > 0 \) (and thus \( v^* \leq 0 \))
	whenever \( \lambda > \frac{99}{9.99} \approx 10 \) and
	\( f(v_+^*) - f(v_-^*) > 0 < 0 \) otherwise.
	In this latter case, \( v^* \geq 0 \).

	Taking \( \lambda = 10 \), we find
	\[
		f(v_-^*) = 100 - 0.1 = 99.99 < 101 = y_1^2 + y_2^2 = f(0),
	\]
	so that \( v_-^* \) is optimal and \( v_-^* < 0 \).
	Moreover, \( v_-^* \) is strictly increasing as a function of \( \lambda \),
	for all \( \lambda > \frac{99}{9.99} \)
	so that \( v_-^* \) is optimal and strictly negative on the interval
	\( [\frac{99}{9.99}, 10] \).

	Now, consider \( \lambda = \frac{99}{9.99} - \epsilon \) to see that
	\[
		f(v_+^*) = \frac{990}{9.99} + 1 - 10 \epsilon < 101 = y_1^2 + y_2^2 = f(0),
	\]
	for all \( \epsilon > 0 \).
	We deduce that \( v_+^* \) is optimal for all \( \epsilon > 0 \) and
	thus \( v_+^* \) is optimal and strictly positive on the interval
	\( [0, \frac{99}{9.99}] \).

	To summarize, the solution function for this problem is as follows:
	\[
		\solfn(\lambda) =
		\begin{cases}
			\frac{\lambda}{100^2} - 0.01
			 & \mbox{if \( \lambda > \frac{99}{9.99} \)}  \\
			\cbr{\frac{\lambda}{100^2} - 0.01, 0.1 - \frac{\lambda}{100}}
			 & \mbox{if \( \lambda = \frac{99}{9.99} \)}  \\
			0.1 - \frac{\lambda}{100}
			 & \mbox{if \( \lambda < \frac{99}{9.99} \).}
		\end{cases}
	\]
	This point-to-set map is clearly not open:
	for every sequence \( \lambda_k \uparrow \frac{99}{9.99} \),
	\( v_k \in \solfn(\lambda_k) \) implies
	\( \lim_k v_k \neq \frac{\lambda}{100^2} - 0.01 \).
	Moreover, a similar result holds for limits from above.
	Finally, it is clear by inspection that the optimal model fit is not
	unique at \( \lambda = \frac{99}{9.99} \), cannot be continuous
	in the functional sense, and, since it is not open, also fails to be
	continuous in the sense of point-to-set maps.
\end{proof}

\section{Extension to General Losses}\label{app:general-losses}


In this section, we briefly discuss how to extend our results to general loss
functions.
Although we use the least-squares error throughout our derivations, this can be
generalized to a smooth and strictly convex loss function \( L : \R^n \times \R^n \into \R \)
without difficulty.
To do so, consider the more general problem,
\begin{equation}\label{eq:general-gl}
	\begin{aligned}
		p^*(\lambda) \!=\! \min_{w} \,
		 & F_\lambda(w) \!:=\! \half L(X w, y) + \lambda \sum_{\bi \in \calB} \norm{\wi}_2 \\
		 & \quad \text{s.t.} \quad \Ki^\top \wi \leq 0  \text{ for all } \bi \in \calB.
	\end{aligned}
\end{equation}
If \( L \) is strictly convex, then uniqueness of the
optimal model fit \( \hat y(\lambda) = X w \) follows from straightforward
adaption of \cref{lemma:unique-fit-cgl}.
Indeed, the only property of the squared-error used in this lemma is strict
convexity.

Since the model fit is constant in \( \solfn \) and \( L \) is both smooth and strictly convex,
the gradient \( \nabla_w L( X w, y) \), which is given by
\[
	X^\top \nabla_{\hat y} L( \hat y, y),
\]
must also be constant over \( \solfn \).
Thus, it is straightforward to replace the correlation vector \( \ci = \Xbi^\top (y - X w) \) with \( \Xbi^\top \nabla_{\hat y} L(\hat y, y) \)
throughout our derivations.

We form a Lagrange dual problem for CGL for one continuity-type result.
\cref{prop:model-fit-continuity} uses the Lagrange dual to
show that the correlation vector \( \ci \) is the unique solution to a convex
optimization program and applies standard sensitivity results to obtain continuity
of \( \hat y \).
In this same fashion, \( \Xbi^\top \nabla_{\hat y} L(\hat y, y) \) is the
unique solution to a Lagrange dual problem where the dual objective uses
the convex conjugate \( L^* \), rather than the dual of the quadratic penalty.
If \( \nabla_{\hat y} L(\hat y, y) \) is continuous in \( \hat y \), then
this is sufficient to deduce continuity of the model fit using the same
argument.

\section{Additional Experiments}\label{app:additional-experiments}


\begin{figure}[t]
	\centering
	\includegraphics[width=0.7\textwidth]{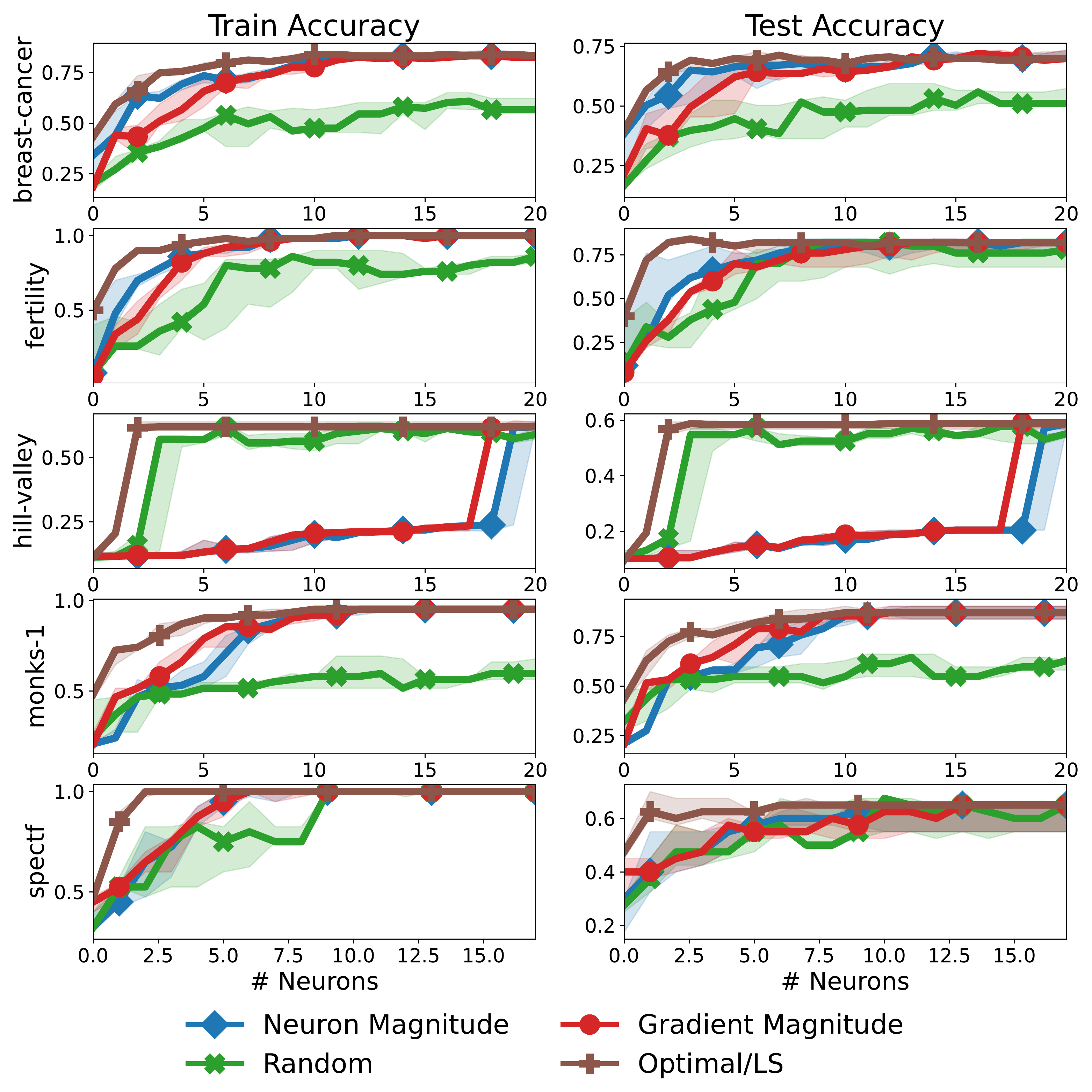}
	\caption{Pruning neurons on five datasets from the UCI repository.
		This figure extends \cref{fig:uci-pruning-acc} with training
		accuracy in addition to the test accuracies shown in the main paper.
	}
	\label{fig:uci-pruning-full}
\end{figure}

\begin{figure}[t]
	\centering
	\includegraphics[width=0.8\textwidth]{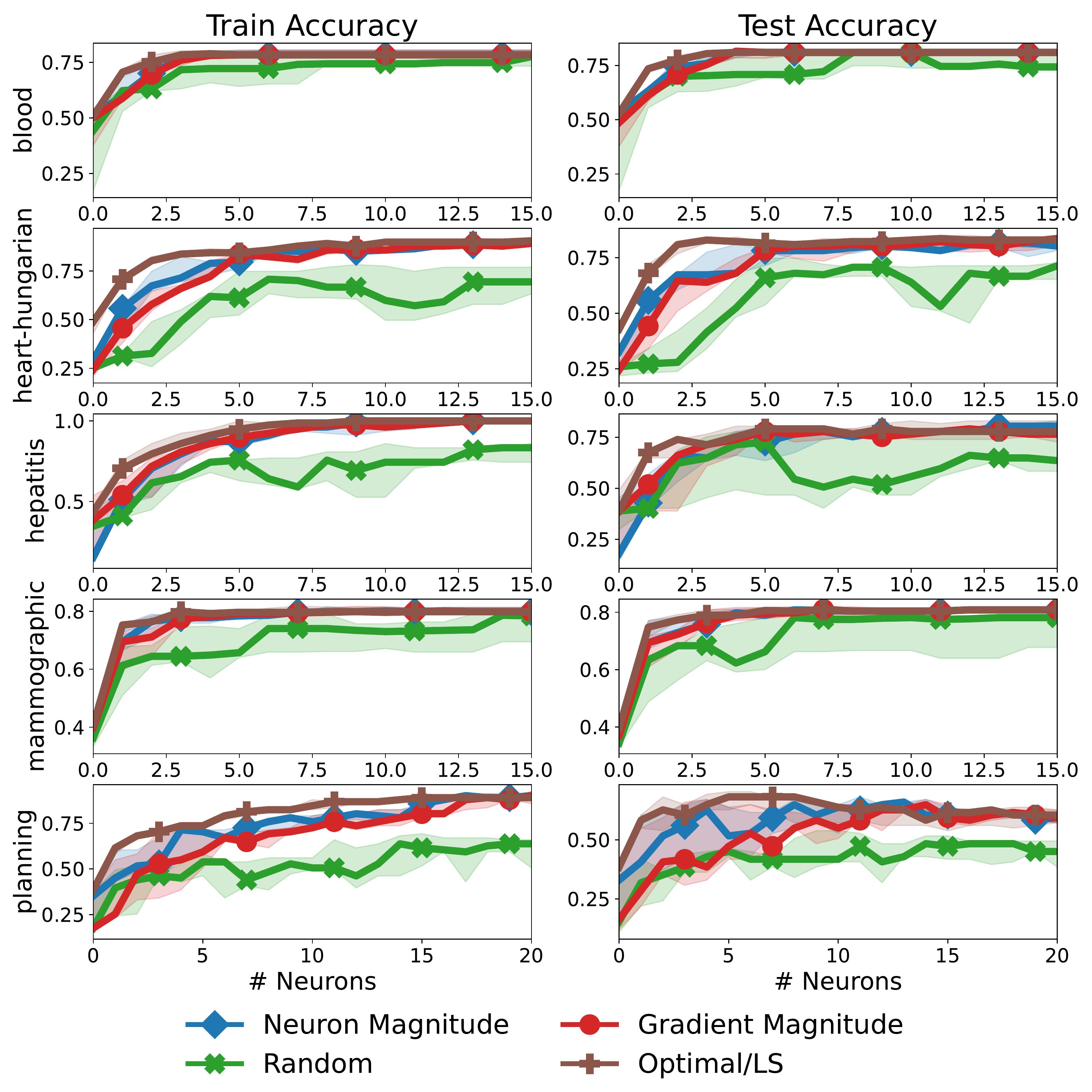}
	\caption{Pruning neurons on five additional datasets from the UCI repository.
		See \cref{fig:uci-pruning-acc} for details.
		Our method (\textbf{Optimal/LS}) preservers test accuracy for longer
		than the baseline methods, leading to compact models with better
		generalization.
	}
	\label{fig:uci-pruning-appendix}
\end{figure}

In this section we provide additional experimental results as well as
the necessary details to replicate our pruning experiments.
Code to replicate all of our experiments is provided at
\url{https://github.com/pilancilab/relu_optimal_sets}.

\subsection{Additional Results}

\begin{figure}[p]
	\centering
	\includegraphics[width=0.8\textwidth]{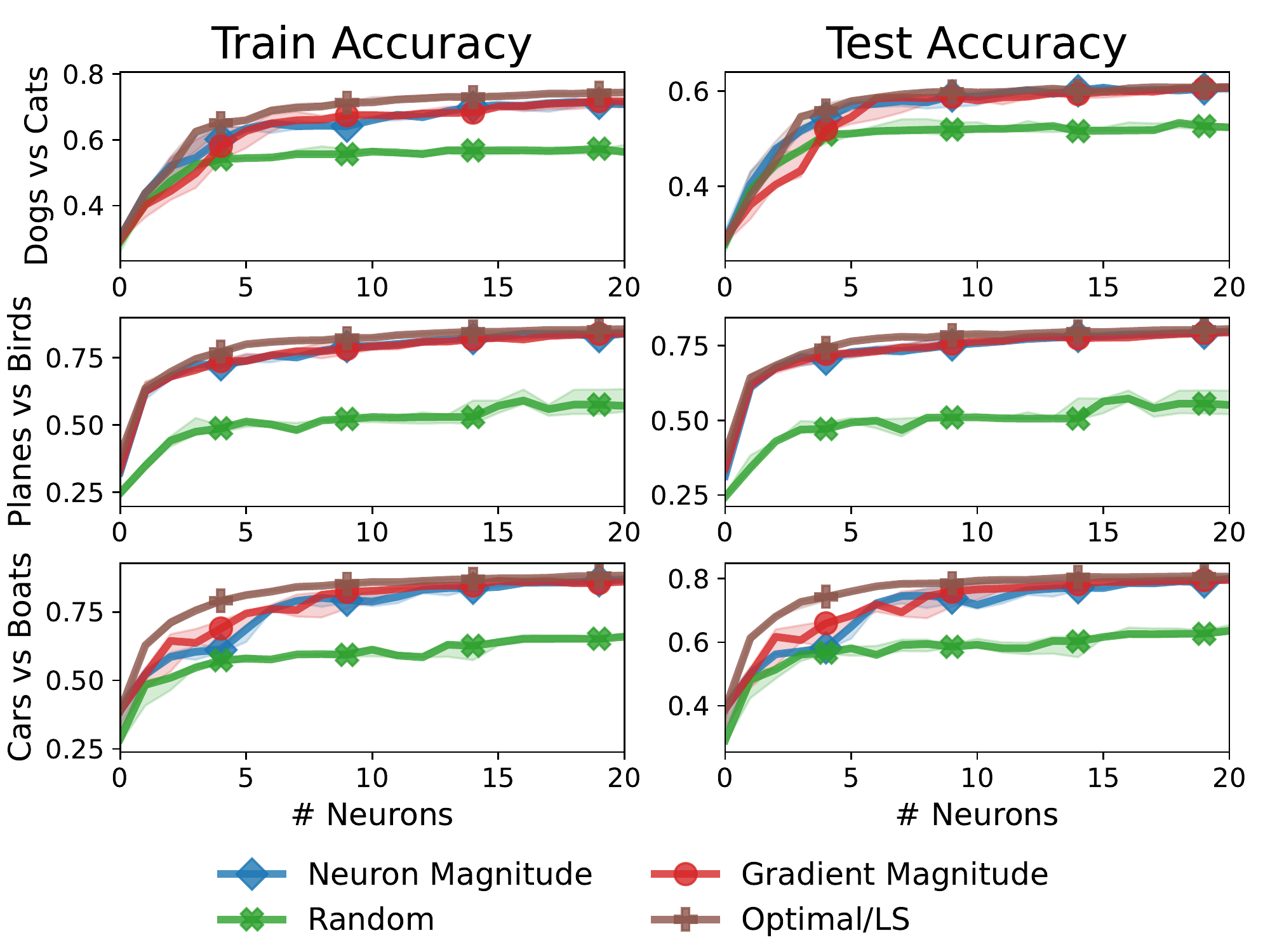}
	\caption{Pruning experiments on binary classification tasks
		from the CIFAR-10 dataset.
		This figure reproduces results from \cref{fig:cifar-pruning-acc}
		with training accuracies added and also includes results for
		an additional task,
		\texttt{cats vs dogs}, not presented in the main paper.
	}
	\label{fig:cifar-pruning-full}
\end{figure}

\begin{figure}[p]
	\centering
	\includegraphics[width=0.8\textwidth]{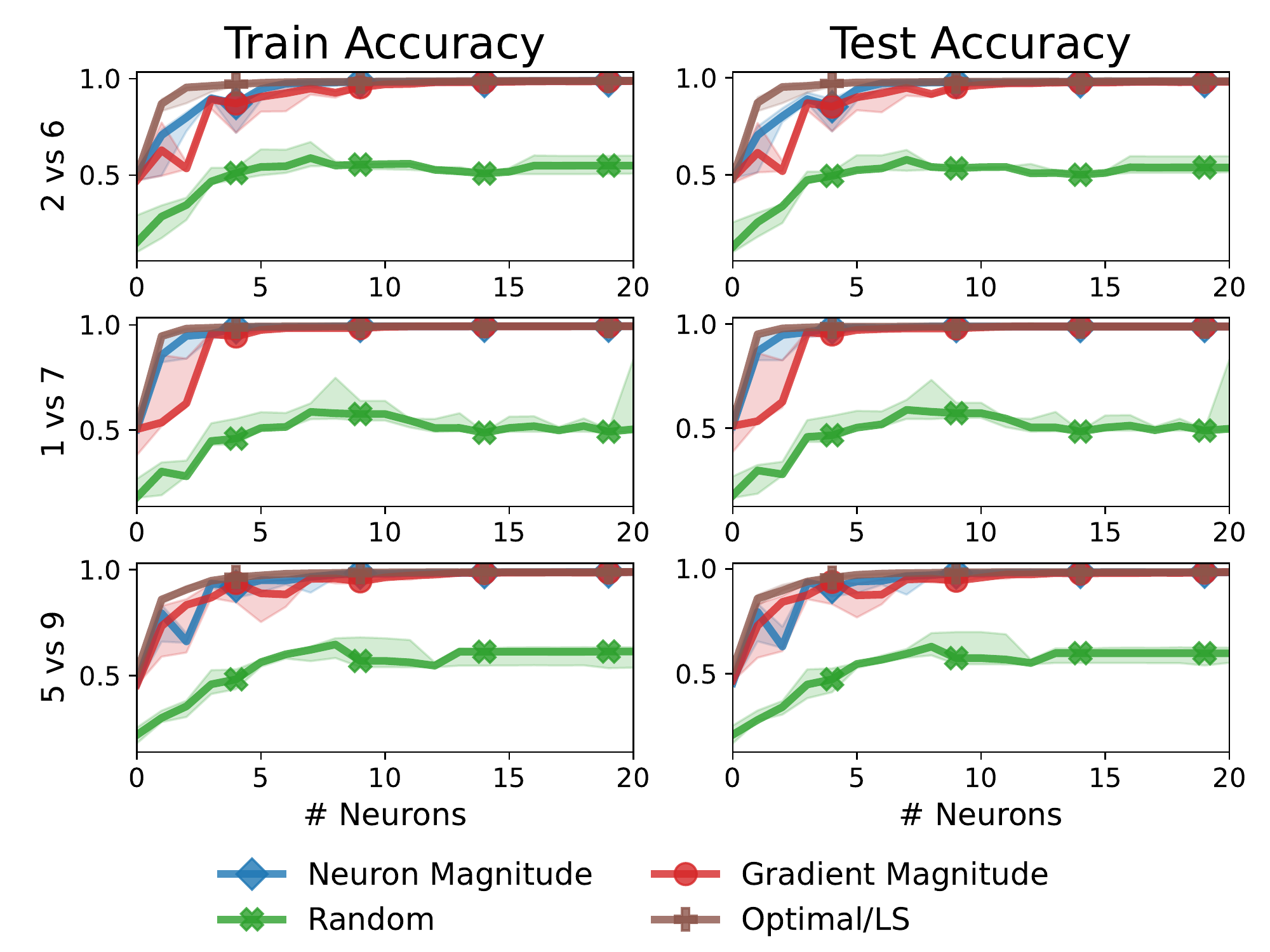}
	\caption{Pruning experiments on three binary classification tasks
		taken from MNIST.
	}
	\label{fig:mnist-pruning-full}
\end{figure}

\begin{table}[t]
	\centering
	\caption{Tuning neural networks by searching over the optimal set.
		We fit two-layer ReLU networks on the training set and
		compute the minimum \( \ell_2 \) norm solution (Min L$_2$).
		Then we tune by finding an extreme point approximating
		the maximum \( \ell_2 \)-norm solution (EP), minimizing validation
		MSE over the optimal set (V-MSE), and minimizing test MSE over
		the optimal set (T-MSE).
		Max Diff. reports the difference between the best and worse models
		found.
		For each method we give the median and interquartile range as
		\texttt{median (lower/upper)}.
	}
	\label{table:tuning-full}
	\vspace{0.1in}
	\begin{tabular}{lccccc} \toprule
		\textbf{Dataset}  & \textbf{Min L$_2$} & \textbf{EP}      & \textbf{V-MSE}   & \textbf{T-MSE}   & \textbf{Max Diff.} \\
		\midrule
		blood             & 0.72 (0.72/0.74)   & 0.72 (0.72/0.74) & 0.62 (0.61/0.62) & 0.7 (0.68/0.71)  & 0.1 (0.11/0.12)    \\
		breast-cancer     & 0.64 (0.61/0.65)   & 0.64 (0.61/0.65) & 0.61 (0.6/0.64)  & 0.71 (0.66/0.71) & 0.1 (0.06/0.08)    \\
		fertility         & 0.66 (0.62/0.7)    & 0.69 (0.62/0.69) & 0.65 (0.64/0.7)  & 0.64 (0.57/0.64) & 0.05 (0.06/0.06)   \\
		heart-hungarian   & 0.75 (0.7/0.77)    & 0.75 (0.7/0.77)  & 0.71 (0.56/0.72) & 0.85 (0.82/0.86) & 0.14 (0.26/0.14)   \\
		hepatitis         & 0.75 (0.74/0.78)   & 0.75 (0.74/0.78) & 0.73 (0.69/0.75) & 0.77 (0.77/0.9)  & 0.05 (0.08/0.15)   \\
		hill-valley       & 0.64 (0.64/0.65)   & 0.65 (0.64/0.65) & 0.64 (0.64/0.67) & 0.64 (0.64/0.65) & 0.0 (0.0/0.01)     \\
		mammographic      & 0.77 (0.77/0.77)   & 0.77 (0.77/0.77) & 0.57 (0.56/0.62) & 0.78 (0.78/0.8)  & 0.21 (0.22/0.18)   \\
		monks-1           & 0.67 (0.64/0.71)   & 0.66 (0.64/0.71) & 0.49 (0.48/0.51) & 0.57 (0.51/0.61) & 0.17 (0.15/0.2)    \\
		planning          & 0.53 (0.51/0.61)   & 0.52 (0.51/0.61) & 0.53 (0.52/0.53) & 0.7 (0.68/0.74)  & 0.17 (0.17/0.21)   \\
		spectf            & 0.64 (0.62/0.7)    & 0.64 (0.62/0.7)  & 0.56 (0.53/0.56) & 0.58 (0.56/0.66) & 0.08 (0.09/0.14)   \\
		horse-colic       & 0.75 (0.75/0.76)   & 0.59 (0.57/0.61) & 0.74 (0.73/0.75) & 0.85 (0.85/0.85) & 0.26 (0.27/0.24)   \\
		ilpd-indian-liver & 0.59 (0.57/0.6)    & 0.59 (0.57/0.6)  & 0.53 (0.53/0.57) & 0.72 (0.7/0.73)  & 0.19 (0.17/0.17)   \\
		parkinsons        & 0.74 (0.72/0.74)   & 0.74 (0.72/0.74) & 0.65 (0.65/0.74) & 0.88 (0.86/0.9)  & 0.23 (0.21/0.16)   \\
		pima              & 0.68 (0.66/0.68)   & 0.68 (0.66/0.68) & 0.68 (0.64/0.7)  & 0.87 (0.86/0.88) & 0.2 (0.22/0.19)    \\
		tic-tac-toe       & 0.98 (0.98/0.98)   & 0.76 (0.69/0.8)  & 0.98 (0.98/0.99) & 1.0 (1.0/1.0)    & 0.24 (0.31/0.2)    \\
		statlog-heart     & 0.71 (0.7/0.73)    & 0.71 (0.7/0.73)  & 0.7 (0.67/0.73)  & 0.84 (0.83/0.86) & 0.14 (0.17/0.13)   \\
		ionosphere        & 0.85 (0.83/0.86)   & 0.76 (0.73/0.76) & 0.84 (0.84/0.84) & 0.88 (0.88/0.89) & 0.12 (0.15/0.12)   \\
		\bottomrule
	\end{tabular}
\end{table}

\textbf{Tuning} \cref{table:tuning-full} shows results for our tuning
task on an additional 7 datasets, as well as the 10 given in \cref{table:tuning-small}.
We report the interquartile range as well as median test accuracies for each
method.
We observe similar results as presented in the main text.
Only one dataset (\texttt{tic-tac-toe}) shows no variation in test
accuracy as we explore the optimal set.

\textbf{Pruning}: \cref{fig:uci-pruning-full} shows train and test
accuracy for our
optimal/least-squares pruning method as well as magnitude/gradient-based pruning
and random pruning on the same five datasets from the UCI repository as
presented in \cref{fig:uci-pruning-acc}.
Our approach shows significantly less decay in train accuracy as neurons are
pruned; this matches the intuition of the least-squares heuristic for pruning,
which selects the coefficients \( \beta \) to best preserve the model fit.

We observe that both our pruning method and pruning by neuron/gradient norm
show very similar training accuracy until most of the neurons have been pruned.
While this behavior is expected from our theory-based approach, it is somewhat
surprising that pruning by neuron/gradient-norm also maintains train accuracy
nearly as well.
This behavior suggests that there are many neurons with very small norm which
can be eliminated without significantly affecting the model prediction.

\cref{fig:uci-pruning-appendix} presents results for neuron pruning on
five additional datasets from the UCI repository, while
\cref{fig:mnist-pruning-full} shows results for three binary classification
tasks taken from the MNIST dataset.
The trends are generally the same as in \cref{fig:uci-pruning-full},
with our approach (\textbf{Theory/LS}) outperforming the baselines.
Finally \cref{fig:cifar-pruning-full} extends the results on CIFAR-10
given in \cref{fig:cifar-pruning-acc} with one additional task
and with training accuracies.

\subsection{Experimental Details}

Now we give the details necessary to reproduce our experiments.
Our experiments use the pre-processed versions of UCI datasets provided by
\citet{delgado2014do}, but do we do not use their evaluation procedure as
it is known to have data leakage.

\subsubsection{Tuning}

We select 17 binary classification datasets from the UCI
repository.
For each dataset we use a random 60/20/20 split of the data
into train, validation, and test sets.
We use the commercial interior point method MOSEK \citep{aps2022mosek}
through the interface provided by CVXPY \citep{diamond2016cvxpy}
to compute the initial model which is then tuned.
We modify the tolerances of this method to use \( \tau = 10^{-8} \) for measuring
both primal convergence and violation of the constraints.
For each dataset, we use fixed \( \lambda = 0.001 \) and a maximum of 1000
neurons.
To compute the min \( \ell_2 \)-norm optimal model,
we use the MOSEK and the optimization problem given in \cref{prop:min-norm-program}.

To approximate the maximum \( \ell_2 \)-norm model, we solve the
following program:
\begin{equation*}
	\begin{aligned}
		\alpha^* = \argmax_{\alpha \geq 0} \sum_{\bi \in \calS_\lambda} \alpha_\bi
		\, \, \text{ s.t.} \, \,
		\sum_{\bi \in \calS_\lambda} \alpha_{\bi} \Xbi \vi = \hat y.
	\end{aligned}
\end{equation*}
This is a linear program which is straightforward to solve using
interior point methods.
Moreover, we have
\[
	\norm{w^*}_2^2 = \lambda^2 \sum_{\bi} (\alpha_\bi^*)^2,
\]
so that \( \sum_{\bi} \alpha_\bi \) acts as an approximation, where we recall
\( \alpha_\bi \geq 0 \) necessarily.

To tune each model with respect to the validation/test MSE, we solve the following optimization problem:
\[
	\min_{w} \cbr{ \half \norm{\tilde X w - \tilde y}_2^2 : w \in \solfn(\lambda) },
\]
with respect to the parameters of the convex formulation.
Here, \( (\tilde X, \tilde y) \) is either the validation or test
set.
We repeat each experiment five times with different random splits of the
data and random resamplings of \( 500 \) activation patterns \( D_i \).
This guarantees that each non-convex network has at most \( 1000 \) neurons
after optimization, although it may have less due to the sparsity inducing
penalty.

\subsubsection{Pruning}

\textbf{Methods}:
We use the augmented Lagrangian method of \citet{mishkin2022convex} to compute
the starting model which is then pruned.
We modify the tolerances of this method to use \( \tau = 10^{-8} \) for measuring
both primal convergence and violation of the constraints.

Pruning by neuron magnitude is straightforward: we sort the neurons
by \( \norm{W_{1i} w_{2i}}_2 \), which measures the total magnitude of the
neuron, and then drop the smallest one.
For pruning by gradient norm, we compute
\( G_{1i} = \nabla_{W_{1i}} L(f_{W_1, w_2}(Z), y) \),
\( g_{2i} = \nabla_{W_{2i}} L(f_{W_1, w_2}(Z), y) \) and then score each
neuron as
\[
	s_i = \norm{W_{1i} \cdot G_{1i} w_{2i} g_{2i}}_2,
\]
where \( \cdot \) indicates the element-wise product.
The neuron with smallest score is zeroed.
This is consistent the existing implementations of pruning by
gradient norm \citep{bialock2020pruning} and attempts to measure the
variation of a linearization of the loss in neuron \( i \).
For Random, we simply select a neuron from a uniform random distribution.

For Optimal/LS, we start by using \cref{alg:pruning-solutions} to prune
until no linear dependence exists amongst the neuron fits \( D_i Z W_{1i} \).
At this point, we choose \( \beta \) to minimize the squared-error
in the training fit.
We choose the neuron to prune by selecting the index that minimizing the residual
in the least-squares fit,
\[
	i_k = \argmin_{j} \min_{\beta}
	\half \norm{\sum_{i \neq j} \beta_i D_i Z w_i  - D_j Z w_j}.
\]
This produced the best result in all of our experiments, although you can
also select \( i_k \) using neuron magnitude or any other rule in the literature
on structured pruning.

\textbf{UCI Datasets}:
We select 10 moderately-sized binary classification datasets from the UCI repository.
For each dataset we use a random 50/50 split of the data
into train and test sets,
fixed \( \lambda = 0.01 \), and sample 25
activation patterns \( D_i \).
This results in a maximum of 50 neurons in each final model;
since the datasets are low dimensional, randomly sampling activation patterns
typically results in fewer than 50 neurons.
All results are repeated for five different random splits and we
plot the median and interquartile ranges of the results.

\textbf{MNIST and CIFAR-10}:
We select three binary classification tasks from each dataset such that
no task shares a target with another task.
For each dataset we use a random 50/50 split of the data
into train and test sets.
For MNIST, we use \( \lambda - 0.01 \), while we used \( \lambda = 0.05 \)
for CIFAR-10.
We sample 50 activation patterns \( D_i \) for each tasking,
which produces a maximum of 100 neurons in each final model.
All results are repeated for five different random splits and we
plot the median and interquartile ranges of the results.

\end{document}